\pdfoutput=1

\documentclass[11pt]{article}

\usepackage{acl}

\usepackage{times}
\usepackage{latexsym}
\usepackage{booktabs}

\usepackage[T1]{fontenc}

\usepackage[utf8]{inputenc}

\usepackage{microtype}

\usepackage{inconsolata}

\usepackage{amsthm}
\usepackage{amsmath}
\usepackage{amssymb}
\usepackage{mathtools}
\usepackage{paralist}
\usepackage[makeroom]{cancel}
\usepackage{tikz}

\usepackage{cleveref}

\usepackage{stmaryrd}

\newtheorem{proposition}{Proposition}

\theoremstyle{definition}

\usepackage{xcolor}

\DeclareMathOperator*{\EX}{\mathbb{E}}%

\let\mathexp=\exp %
\usepackage{gb4e}
\noautomath %
\let\gbexp=\exp %

\DeclareRobustCommand{\exp}{\ifmmode\mathexp\else\expandafter\gbexp\fi}

\Crefname{xnumi}{Example}{Examples}
\Crefname{xnumii}{Example}{Examples}

\DeclarePairedDelimiter\abs{\lvert}{\rvert}%
\DeclarePairedDelimiter\norm{\lVert}{\rVert}%

\makeatletter
\let\oldabs\abs
\def\abs{\@ifstar{\oldabs}{\oldabs*}}

\let\oldnorm\norm
\def\norm{\@ifstar{\oldnorm}{\oldnorm*}}
\makeatother

\usepackage{pifont}  %
\newcommand{\cmark}{\textcolor[HTML]{009900}{\ding{51}}}%
\newcommand{\xmark}{\textcolor[HTML]{900000}{\ding{55}}}%
\newcommand\marked[1]{\emph{\color{gray}*#1}}
\newcommand\entailed[1]{{\textcolor[HTML]{009900}{$\Rightarrow$#1}}}

\Crefformat{section}{\S#2#1#3}
\Crefformat{appendix}{\S#2#1#3}
\Crefformat{subsection}{\S#2#1#3}
\Crefformat{subapp}{\S#2#1#3}

\definecolor{olmoBlue}{HTML}{265ed4}
\newcommand{\balpha}{{\color{olmoBlue} \boldsymbol{\alpha}}}
\newcommand{\bbeta}{{\color{olmoBlue} \boldsymbol{\beta}}}

\title{Can You Learn Semantics Through Next-Word Prediction?\\The Case of Entailment}
\author{
    William Merrill\thanks{~~Equal contribution. We release our code and data at\\
    \href{https://github.com/ZhaofengWu/entailment-from-lm}{\nolinkurl{github.com/ZhaofengWu/entailment-from-lm}}.
    }~$\hspace{0.5mm}^{\balpha}$ \quad
    Zhaofeng Wu$^{*\bbeta}$ \quad
    Norihito Naka$^\balpha$ \quad
    Yoon Kim$^\bbeta$ \quad
    Tal Linzen$^\balpha$ \\
    $^\balpha$New York University \quad
    $^\bbeta$Massachusetts Institute of Technology
}

\begin{document}

\maketitle

\begin{abstract}
    Do LMs infer the semantics of text from co-occurrence patterns in their training data? \citet{merrill-etal-2022-entailment} argue that, in theory, sentence co-occurrence probabilities predicted by an optimal LM should reflect the entailment relationship of the constituent sentences, but it is unclear whether probabilities predicted by neural LMs encode entailment in this way because of strong assumptions made by \citeauthor{merrill-etal-2022-entailment} (namely, that humans always avoid redundancy). In this work, we investigate whether their theory can be used to decode entailment relations from neural LMs. We find that a test similar to theirs can decode entailment relations between natural sentences, well above random chance, though not perfectly, across many datasets and LMs. This suggests LMs implicitly model aspects of semantics to predict semantic effects on sentence co-occurrence patterns. However, we find the test that predicts entailment in practice works in the opposite direction to the theoretical test. We thus revisit the assumptions underlying the original test, finding its derivation did not adequately account for redundancy in human-written text. We argue that better accounting for redundancy related to \emph{explanations} might derive the observed flipped test and, more generally, improve computational models of speakers in linguistics.
\end{abstract}

\section{Introduction}

Inspired by the empirical capabilities of language models (LMs) trained on next-word prediction, recent work has examined if and how linguistic meaning might be inferred from raw text \citep[][inter alia]{bender-koller-2020-climbing,merrill2021provable,pavlick2022semantic,wu-etal-2023-transparency}.
A text corpus is the result of humans using text to communicate information, and doing this efficiently requires following pragmatic principles like avoiding contradictory or redundant sentences.
Therefore, training to predict whether sentences can co-occur (which can reduce to next-token prediction) might lead LMs to represent semantic relationships between sentences \citep{harris1954distributional,potts2020possible,michael2020dissect}.

But does sentence co-occurrence provide enough signal for LMs to learn to represent complex semantic phenomena like entailment?
\citet{merrill-etal-2022-entailment} derive a simple equation by which the entailment relation between two sentences can be detected using their co-occurrence probability in a corpus generated by speakers who avoid redundancy. Intuitively, non-redundant speakers will rarely utter entailed sentences, so low co-occurrence probability of two sentences is predictive of their entailment relationship.
This means that, in principle, learning to model sentence co-occurrence perfectly requires an LM to implicitly model entailment,
and entailment classifications can be extracted from the co-occurrence probabilities of such an LM.

However, \citeauthor{merrill-etal-2022-entailment}'s theoretical result has two caveats.
First, it assumes an ``ideal'' LM that perfectly models the likelihood of texts in a language. 
Second, it makes the strong (but theoretically motivated; \citealp{grice1975logic}) assumption that speakers always avoid redundancy.
It is thus unclear whether real LMs infer a model of entailment from sentence co-occurrence probabilities in their training data, both because LMs may misestimate probabilities and because the required assumptions about human speakers may be too simplified.

In this work, we empirically evaluate the distributional entailment test from \citet{merrill-etal-2022-entailment}: can we use it to classify entailment from LM probability estimates?
Overall, we find across a wide range of entailment benchmarks and LMs that a variant of the entailment test consistently detects entailment well above random chance. 
This suggests that LM probability judgments are sensitive to the relationships between sentence meanings that are reflected in sentence co-occurrence patterns, at least to some extent. This further suggests that next-word prediction is a strong enough objective for LMs to acquire at least a partial model of entailment relationships between sentences.

However, this result comes with a surprise. Across many entailment benchmarks, we find that the direction of the test is \emph{flipped} compared to \citeauthor{merrill-etal-2022-entailment}'s theoretical test: higher co-occurrence probabilities correlate with entailment when the \emph{opposite} is expected!
We take this as evidence against a theory of human speakers based purely on minimizing redundancy.
Analyzing natural corpora, we find humans are often more redundant than \citeauthor{merrill-etal-2022-entailment}'s non-redundant speakers, which could explain the flipped test.
We present a preliminary account of how better accounting for explanations (one observed type of redundancy) might predict the flipped test.
Overall, our results motivate future work in computational pragmatics accounting for redundancy and are a case study for how the data aggregated about many speakers in LMs can be used to test and develop pragmatic theories.

\section{Distributional Semantics and the Entailment Test} \label{sec:bakground}
There is an old debate in linguistics and NLP about whether distributional semantics---the idea that text co-occurrence patterns can contain semantic information---captures semantics in any true sense \citep{brunila-laviolette-2022-company}. This debate goes back at least to \citet{harris1954distributional}, who argues that sentence co-occurrences patterns in a corpus could be used as data to build a linguistic theory of semantics, but it has been revisited in recent years in terms of LMs. In particular, \citet{bender-koller-2020-climbing}---in disagreement with \citet{harris1954distributional}---took a strong stance against the claim that LMs ``understand'' language because understanding requires modeling communicative intent or at least conventionalized semantic denotations, both of which do not appear explicitly in the training data for LMs.

While it is certainly true that LMs are only trained on surface forms,
counterarguments to \citet{bender-koller-2020-climbing} have been given for how LMs might be able to reconstruct semantic information from their training data.
One line of counterarguments \citep{potts2020possible,michael2020dissect,merrill-etal-2022-entailment} echoes \citet{harris1954distributional}, positing that sentence co-occurrence probabilities contain information about semantics because speakers aim to be truthful and informative and are thus unlikely to produce contradictory or redundant pairs of sentences. Properly learning which sentences can co-occur (part of LM training) thus amounts to acquiring a \emph{semantic} representation of which sentences are contradictory or redundant with one another.\footnote{Alternative signals also exist that LMs could use to bootstrap a semantic representation, such as assertions~\citep{merrill2021provable}.}
\citet[][CoNLL slides]{merrill-etal-2022-entailment} motivate this claim
with the following example:
\begin{exe}
    \ex I have two cats. \label{ex:cats}
    \begin{xlist}
        \ex \marked{I don't have a cat.} \label{ex:dont-have}
        \ex \marked{I have a cat.} \label{ex:have}
        \ex One is orange. \label{ex:orange}
    \end{xlist}
\end{exe}
\Cref{ex:dont-have} is unlikely to be uttered because it contains a contradiction. More subtly, \Cref{ex:have} is unlikely because its second sentence is uninformative given the first, even though they are consistent. In contrast, \Cref{ex:orange} is acceptable because it is consistent \emph{and} adds new information.
Thus, \Cref{ex:cats} suggests sentence co-occurrence is governed by semantic constraints against inconsistency and redundancy. If strong LMs correctly model such co-occurrences, they might need an implicit model of sentence semantics to determine these properties.

\subsection{The Entailment Test}

One way to define semantic competency is the ability to resolve entailment relations between pairs of sentences.
This simple idea has a long history both in both the philosophy of language \citep{VanBenthem1986,brandom2000articulating} and NLP evaluation \citep{dagan2010recognizing}.
Drawing on the semantic nature of sentence co-occurrence and its connection to redundancy,
\citet{merrill-etal-2022-entailment} derive a test to check whether sentence $x$ entails sentence $y$ using their co-occurrence probability in a corpus produced by so-called \emph{Gricean speakers}.
If we accept the idea that the ability to evaluate entailment captures semantics in full, this test establishes semantics, can, in principle, be inferred from next-word prediction. 

\paragraph{Gricean Speakers.}
Gricean speakers are a computational model of human speakers implementing principles for effective communication (the Gricean maxims; \citealp{grice1975logic}). The maxims say a speaker should convey as much relevant information as possible without saying too much, among other desiderata.
Following standard computational choices in rational theories of pragmatics \citep{goodman2016pragmatic},
\citet{merrill-etal-2022-entailment} operationalize the maxims by modeling the probability of a text $z$ produced by a Gricean speaker as a function of $z$'s information content and cost:
\begin{compactitem}
    \item \textbf{Information content:}
    Sentences that convey more information to a listener are more likely to appear in a corpus than those that convey less.
    This penalizes untruthful, uninformative, and redundant sentences.
    Let $i_\ell(y \mid x, w)$ be the information $y$ conveys to the listener given beliefs $w$ and context $x$, which speakers aim to \emph{maximize}.
    \item \textbf{Cost:} Long or complex sentences should be less likely so that speakers do not produce informative, but verbose, text. The model assumes a function $c(y)$ that gives the cost of sentence $y$, which speakers aim to \emph{minimize}.
\end{compactitem}
Under \citet{merrill-etal-2022-entailment}'s model, a Gricean speaker utters $y$ (having said $x$) with probability
\begin{equation*}
    p(y \mid x, w) \propto \exp(i_\ell(y \mid x, w) - c(y)) .
\end{equation*}
A sequence of sentences $z_1 \cdots z_n$ occurs in a corpus generated by Gricean speakers with probability
\begin{equation*}
    p(z) = \EX_w \left[ \prod_{i=1}^n p(z_i \mid z_{<i}, w) \right] .
\end{equation*}
Let $\$$ denote a special ``end-of-text'' sentence.

\paragraph{Entailment Test.} Assuming a corpus is sampled from a collection of Gricean speakers with different beliefs, \citet{merrill-etal-2022-entailment} derive the following measure $\hat E_p(x, y)$ for detecting entailment purely using log probabilities of sentence co-occurrences:
\begin{align} \label{eq:orig-test}
\begin{split}
\hat E_p(x, y)
    &= \log p(xy) - \log p(x\$) \\
    &- \log p(yy) + \log p(y\$).
\end{split}
\end{align}

\noindent A ${\sim}0$ score means entailment.
The first two terms $\approx \log p(y \mid x)$ and the last two $\approx {-}\log p(y \mid y)$. This gives some intuition for the test: 0 means $xy$ is as redundant as $yy$, i.e., $x$ entails $y$ (see \Cref{sec:test-details}).

\section{Evaluating the Entailment Test}

\citet{merrill-etal-2022-entailment} showed their test could detect entailment from n-gram LMs trained on synthetic data generated by Gricean speakers. Although Gricean speakers capture some principles of how humans speak, they are likely simplistic compared to real language use.
Additionally, real LMs may misestimate the co-occurrence probabilities used by the test.
For both of these reasons, it is unclear whether the entailment test should correctly detect entailment on natural sentences given LM-estimated probabilities.
We thus evaluate the entailment test with probabilities computed by real LMs on natural-language entailment benchmarks.

\subsection{Entailment Datasets}

We first evaluate the entailment test on existing \emph{broad-coverage} entailment datasets built by crowd workers: RTE~\citep{dagan2010recognizing}, MNLI~\citep{williams2018mnli}, WaNLI~\citep{liu-etal-2022-wanli}, and ANLI~\citep{nie-etal-2020-adversarial}.\footnote{For ANLI, we use the data collected in the third round.} Unless otherwise mentioned, we always use the training set.
We collapse three-way label distinctions (entailment, neutral, contradiction) to entailment or non-entailment.
We also evaluate on \emph{targeted synthetic} entailment datasets designed to test specific kinds of entailment
à la GLUE diagnostics~\citep{wang-etal-2018-glue}: specifically, entailment related to
the logical connectives \emph{and}/\emph{or}, the quantifiers \emph{all}/\emph{some}, numbers, passivization, and datives (details in \Cref{sec:targeted-eval}).
We reported dataset statistics in \Cref{sec:dataset-stats}.

\subsection{Models}

We evaluate the entailment test with probabilities computed by a diverse suite of LMs: GPT-2 \citep{Radford2019LanguageMA}, OPT \citep{zhang2022opt}, \mbox{Llama-1} \citep{touvron2023llama}, Vicuna \citep{vicuna2023}, Llama-2, and Llama-2-Chat~\citep{touvron2023llama2}. The LMs vary in size, pretraining data, and whether and how they undergo an ``alignment'' process (i.e., instruction-tuning or RLHF). For each LM family, we use both the smallest and the largest publicly available LM
(see \Cref{sec:lms-we-used} for a list).

\subsection{Evaluation Metric: Flipped ROC-AUC}

The entailment test does not directly classify entailment but gives a score where ${\sim}0$ suggests entailment and higher values suggest non-entailment.
This can be converted to a classifier by choosing a decision boundary for entailment, but the choice of a threshold is arbitrary.
To evaluate the test, we thus use the standard ROC-AUC metric, which can be understood to evaluate the score holistically across different choices of the threshold. There is an inherent tradeoff between precision and recall with the choice of the threshold, and ROC-AUC provides a consistent way to evaluate without arbitrarily fixing the threshold. Independent of the class imbalance, ROC-AUC ranges from 0 to 100 where 50 is random chance.
In many cases, we found that the flipped entailment score (meaning \Cref{eq:orig-test} with the sign of each term flipped) detected entailment better than the original score (\S\ref{sec:test-is-flipped}).
We thus report the ROC-AUC score of the flipped score, which we call \emph{flipped ROC-AUC}.

\begin{figure}[t!]
    \centering
    \includegraphics[width=0.42\textwidth]{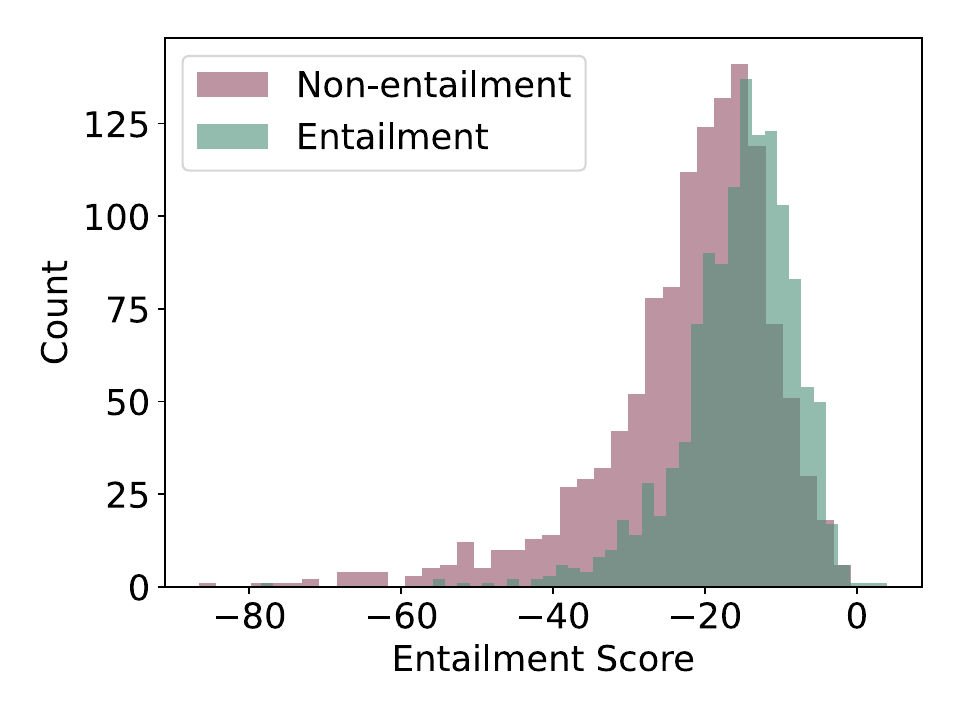}
    \vspace{-2mm}
    \caption{Entailment score $\hat E_p(x, y)$ distribution computed with Llama2-70b probabilities on RTE. \textbf{The score discriminates the two classes, though imperfectly.}}
    \label{fig:density}
    \vspace{-4mm}
\end{figure}

\section{Entailment Test Results} \label{sec:base-test-results}

Overall, we find the test predicts entailment on the broad-coverage datasets, but only when the test score is \emph{flipped} compared to the theoretical test (i.e., a larger score means entailment).
However, the pattern is more complicated for the targeted data, where some constructions follow the flipped trend but others follow the original, unflipped test.

\subsection{Flipped Test on Broad-Coverage Data} \label{sec:test-is-flipped}

\Cref{fig:density} shows the entailment score $\hat E_p(x, y)$ for the RTE training data using Llama2-70b probabilities. The score distinguishes the two classes, but not perfectly. However, the theory predicts smaller $\hat E_p(x, y)$ for entailment vs.~non-entailment, which is \emph{flipped} in \Cref{fig:density} (which we try to account for in \Cref{sec:towards-accounting-for-redundancy}).
We find this holds consistently across the broad-coverage datasets: the flipped entailment test detects entailment above random chance and a length baseline that is designed to control for spurious correlations~\citep{gururangan-etal-2018-annotation}\footnote{Computed by using the premise length, the hypothesis length, or the inverse of each, as the score, whichever of the four yields the best flipped AUC-ROC.} (\Cref{fig:aucroc-histogram}).
We also hypothesize the entailment test should be more predictive for better LMs.
Using bits per byte~(BPB; \citealp{gao2020pile})\footnote{To be comparable across tokenizaion schemes.} on the C4 validation set~\citep{t5} as the proxy for model quality, we plot their correlation in \Cref{fig:aucroc-vs-perplexity}.
Across broad-coverage datasets, better (lower) BPB is associated with higher flipped ROC-AUC.
This suggests LMs that more accurately predict the next token also better model sentence co-occurrence patterns reflecting entailment.

\begin{figure}[t!]
    \centering
    \includegraphics[width=0.44\textwidth]{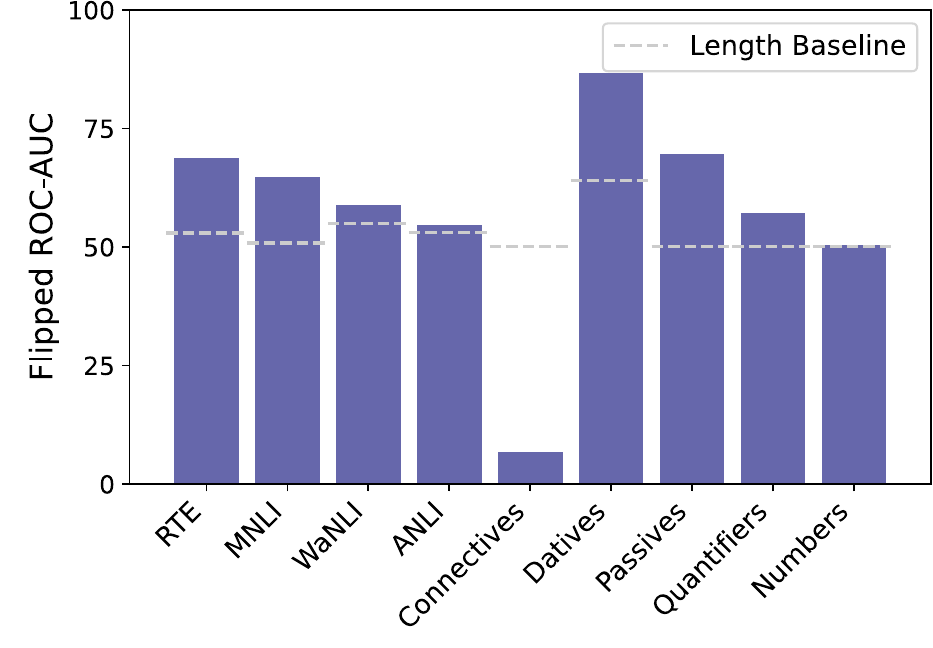}
    \vspace{-2mm}
    \caption{Flipped AUC-ROC scores for the entailment test across datasets using Llama2-70b probabilities. \textbf{The flipped test generally performs above random} (=50) and the length baseline, while the original test works better for connectives ($<$50 Flipped ROC-AUC).}
    \vspace{-2mm}
    \label{fig:aucroc-histogram}
\end{figure}

We also evaluate how test performance emerges during training using Pythia-12b checkpoints.
\Cref{fig:pythia} shows that ROC-AUC consistently increases as training progresses.
Around 1b tokens, flipped ROC-AUC scores on RTE, MNLI, and WaNLI sharply increase together, suggesting the model undergoes a phase transition where general features useful for predicting entailment may be emerging~\citep{chen2024sudden}.

\begin{figure*}[t!]
    \centering
    \includegraphics[width=\textwidth]{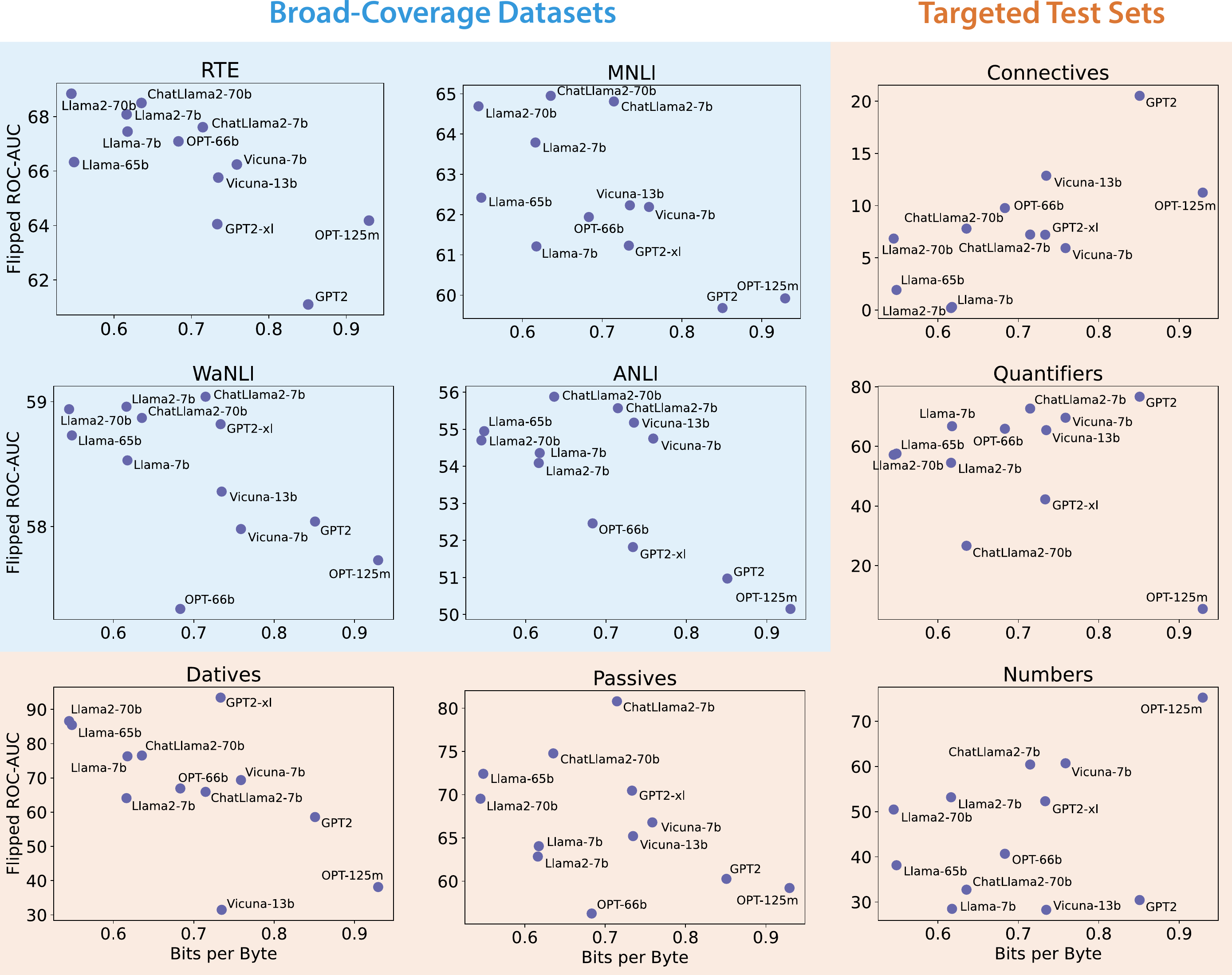}
    \caption{C4 validation bits per byte vs.~flipped AUC-ROC score for all models on broad-coverage
    and targeted datasets. Note that the scale of the $y$-axis differs for each subplot. See \Cref{fig:aucroc-histogram} for a scale-controlled version of Llama2-70b results.
    \textbf{For broad-coverage datasets,  model quality (represented by bits per byte, lower is better) clearly correlates with flipped test performance,} though this is more complicated for the targeted test sets.}
    \label{fig:aucroc-vs-perplexity}
\end{figure*}

\subsection{Varied Pattern for Targeted Phenomena}

\Cref{fig:aucroc-histogram} shows the flipped test works better for datives, passives, and quantifiers. For connectives, the unflipped test better predicts entailment. This suggests that, while the flipped test outperforms the original test in aggregate, the original theory might apply only for \emph{some} constructions.
\Cref{fig:aucroc-vs-perplexity} shows the association between LM BPB and flipped ROC-AUC for the targeted cases.
Datives, passives, and quantifiers show a similar trend to the broad-coverage data where lower BPB associates with higher flipped ROC-AUC, but
connectives and numbers mostly follow the original test.

\subsection{Learning a Distributional Entailment Test} \label{sec:learning-a-distributional-entailment-test}

We have seen that the distributional entailment test of \citet{merrill-etal-2022-entailment} can detect entailment, but only when the sign of each term is flipped.
We now evaluate this flipped test by comparing it to an oracle test that optimally predicts entailment. Their discrepancies would inform us about realistic LMs and data distributions.
We train a small regression model that weights co-occurrence probabilities to predict entailment and inspect the learned weights.

\paragraph{Setup.}
The original entailment test can be viewed as a linear model with features $\phi$ and parameters $\theta$:
\begin{align*}
    \phi &= \langle \underbrace{\log p(xy), \log p(x\$)}_{\textit{Left-hand side (LHS)}}, \underbrace{\log p(yy), \log p(y\$)}_{\textit{Right-hand side (RHS)}} \rangle \\    
    \theta &= \langle 1, -1, -1, 1 \rangle .
\end{align*}

Instead of applying the test with parameters $\theta$ (original test) or $-\theta$ (flipped test), we now \emph{learn} parameters $\hat \theta$ via logistic regression on labeled entailment pairs.
This learned test is \emph{not} a standard supervised text classifier: it only gets sentence co-occurrence log-probabilities as input, not text itself.

\begin{figure}[!t]
    \centering
    \includegraphics[width=\columnwidth]{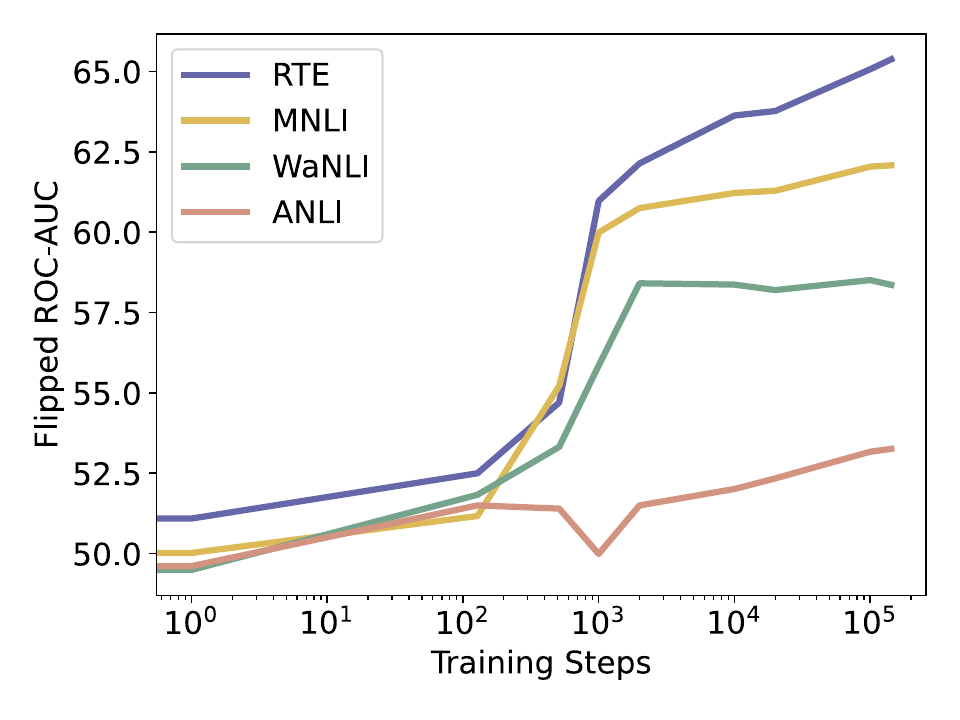}
    \caption{Flipped ROC-AUC of entailment score across Pythia-12b checkpoints. Each step is around 2M tokens.}
    \label{fig:pythia}
\end{figure}

\paragraph{Results.}
\Cref{fig:learned-coefficients-paired} shows the results for the broad-coverage datasets (other datasets in \S\ref{sec:learned-entailment-test-for-more-datasets}).
For the LHS, the negative $xy$ weight matches the positive $x\$$ weight in magnitude, as for the flipped test. For the RHS, the trend is less consistent, but
$yy$ and $y\$$ generally get smaller weights than the LHS terms.
Nevertheless, in aggregate, $yy$ gets a positive weight of the same magnitude as the negative $y\$$ weight (\Cref{fig:learned-coefficients-yy-vs-y}), as for the flipped test.

We interpret the similarity between the flipped and learned tests as evidence for the directional correctness of the flipped test.
The main difference between the learned and flipped tests is that the RHS has smaller weights than the LHS for the learned test.
This may be due to the transformer's learning biases and not the underlying data:
Transformer LMs are prone to in-context copying \citep{olsson2022context} and thus might overestimate $\log p(yy)$. Reduced RHS weights may correct for this.

\begin{figure}[t!]
    \centering
    \includegraphics[width=0.48\textwidth]{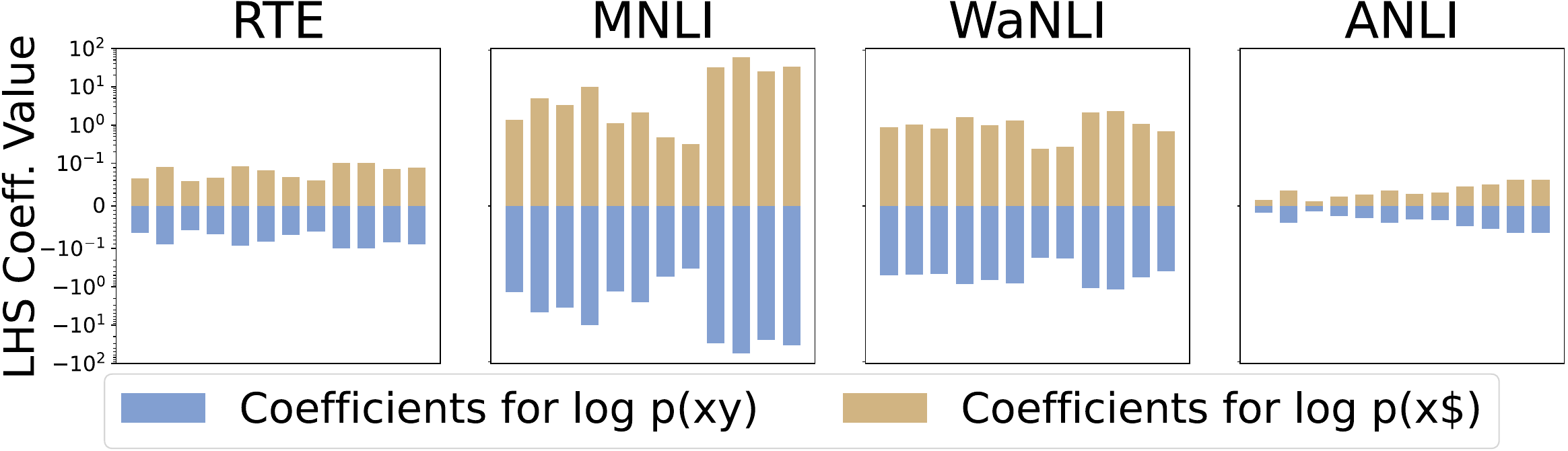}
    \includegraphics[width=0.48\textwidth]{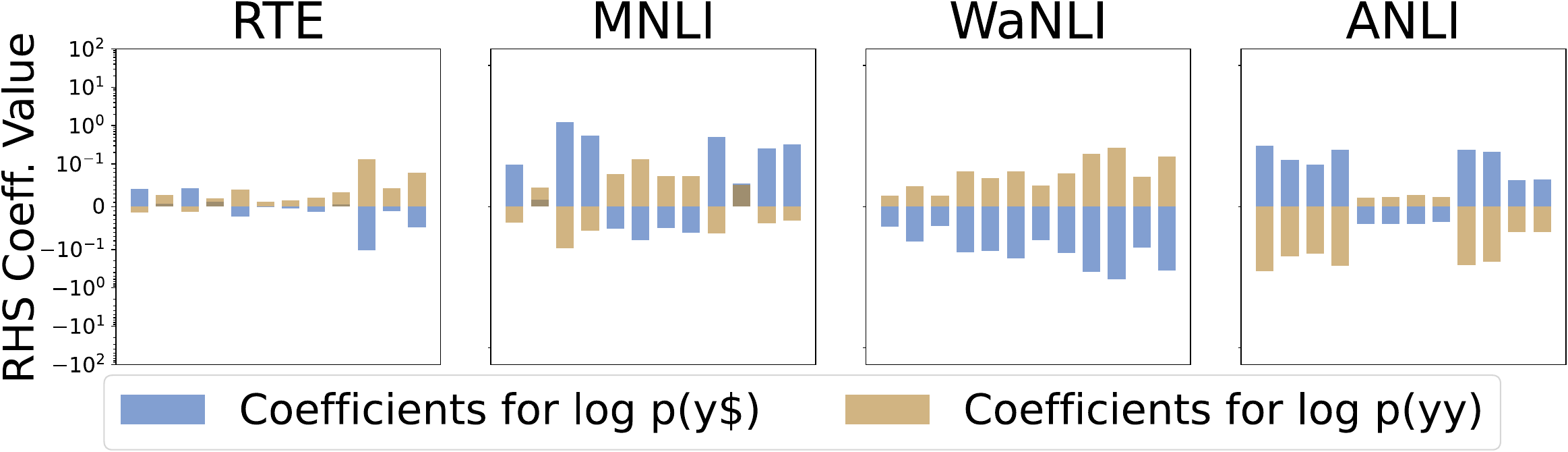}
    \caption{Learned logistic regression coefficients for the log-prob features for the broad-coverage datasets. Each bar represents one LM. For ease of visualization, $y$-axis is in log scale, except in $[-0.1, 0.1]$ where it is linear.
    }
    \label{fig:learned-coefficients-paired}
\end{figure}

\subsection{Results Excluding Contradiction} \label{sec:no-contradiction}

Our results so far have compared the entailment test performance on entailment vs. non-entailment pairs.
However, the most surprising aspect of results (the flipped pattern) involves a comparison of entailment and neutral pairs, as it is expected that contradiction pairs should have a lower score than entailment pairs. 
Thus, in \Cref{sec:no-contradiction}, we repeat all analyses (Figures \ref{fig:aucroc-histogram} to \ref{fig:learned-coefficients-yy-vs-y}) contrasting entailment and neutral pairs, with contradiction excluded. Overall, the results are qualitatively similar, but the correlations between perplexity and test performance is less strong in some cases, and the logistic regression coefficients found on MNLI are less interpretable.

\section{Corpus Study: Characterizing Naturalistic Linguistic Redundancy} \label{sec:corpus-study}

A surprising finding from the previous section is that the entailment test is robustly flipped: entailed continuations tend to be \emph{more likely} than non-entailed ones.
This suggests the Gricean speaker assumed to derive the test may be too simplistic to account for humans.
In particular, we hypothesize the disconnect may be because human speakers are \emph{explicitly redundant} in certain contexts unlike Gricean speakers, who always avoid redundancy.
We thus search for natural instances of \emph{contextually entailed text} in corpora to better understand why real human speakers produce redundant sentences.

\paragraph{Data.}
To find contextually entailed sentences in different types of discourse, we consider a variety of web domains: Book3 \citep{gao2020pile}, Wikipedia (en) \citep{gao2020pile}, Multi-News \citep{fabbri2019multinews} and Reuters-21578 \citep{hayes90construe}, Yahoo! Answers Topics \citep{zhang2016characterlevel}, and Yelp Reviews \citep{zhang2016characterlevel}.

\begin{figure}[t!]
    \centering
    \includegraphics[width=0.48\textwidth]{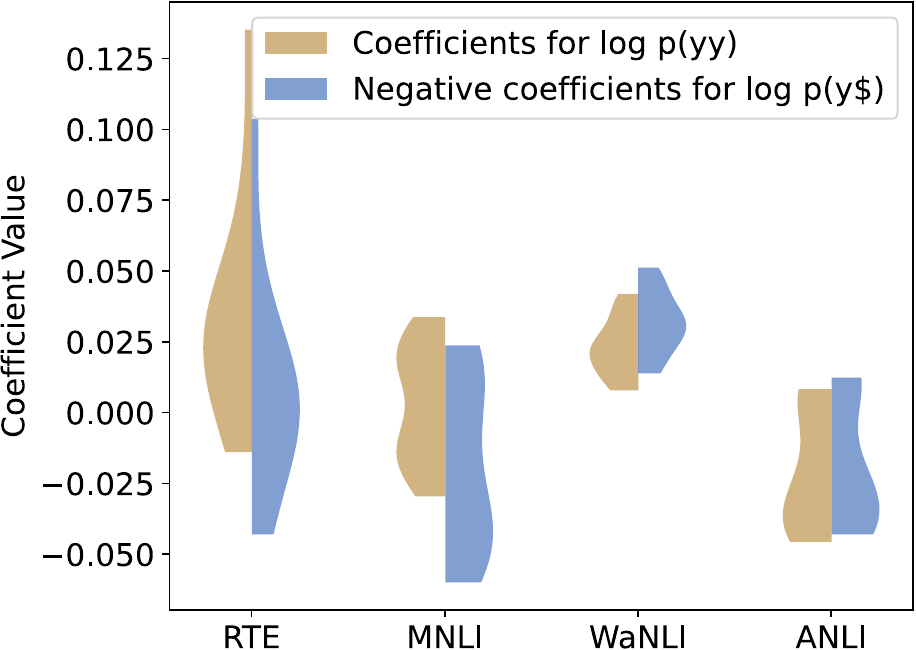}
    \caption{
    The RHS coefficients, for $\log p(y\$)$ and $\log p(yy)$, marginalized across all LMs.
    }
    \label{fig:learned-coefficients-yy-vs-y}
\end{figure}

\paragraph{Finding Contextually Entailed Text.}
For each document in each corpus, we construct premise and hypothesis pairs by choosing six contiguous sentences, with the first five as the premise and sixth as the hypothesis.
We use entailment classifiers finetuned from T5 \citep{honovich-etal-2022-true-evaluating} and RoBERTa \citep{liu2019roberta} to detect entailment pairs
and take the intersection of examples considered entailment by both.
We then manually filter to remove incorrect entailment pairs
(details in \Cref{sec:entailment-manual-inspection}).

\paragraph{Results.}
As \Cref{tab:dataset-entailment} shows, the frequency of entailed sentences is on the order of at least $10^{-3}$.
Even this lower bound is several orders of magnitude higher than expected for a Gricean speaker. Quite conservatively, imagine that for each entailed continuation there is at least one alternative of the same length that conveys 10 nats of information, which is quite reasonable given Shannon's lower bound estimate of 0.4 nats/character\footnote{Technically, the Gricean speaker uses semantic information, whereas Shannon's estimate captures \emph{all} information. However, we imagine most information in text is semantic, so these are on the same order of magnitude.} \citep{shannon1951prediction} and that typical sentences are at least 30 characters.
Then the likelihood of producing an entailed sentence should be at most $1/\exp(10) \approx 10^{-5}$.
This suggests the data cannot be accounted for by assuming speakers always avoid redundancy.

To better understand what is lost when assuming speakers always avoid redundancy, we inspect examples of contextually entailed text from these corpora.
We find there are many reasons speakers produce entailed text.
This includes both \emph{repetition} of previous statements (44.44\%\footnote{Percentages determined manually; see \Cref{sec:entailment-categories} for details.}) and high-level \emph{summaries} or \emph{conclusions}  (35.56\%).
One observed use of repetition is to emphasize an important point:

\newcommand*{\specialcell}[2][b]{%
  \begin{tabular}[#1]{@{}c@{}}#2\end{tabular}%
}
\newcommand*{\specialcellbold}[2][b]{%
  \bfseries
  \begin{tabular}[#1]{@{}c@{}}#2\end{tabular}%
}
\newcommand*{\leftspecialcell}[2][b]{%
  \begin{tabular}[#1]{@{}l@{}}#2\end{tabular}%
}
\begin{table}[]
    \centering
    \begin{tabular}{l|cccc}
        \textbf{Data Sources} & \textbf{T5} & \textbf{RoB} & \textbf{$\cap$} & \textbf{$\cap^{+}$} \\
        \hline
        Book3                & 0.40 & 1.31 & 0.33 & 0.27 \\
        Wikipedia (en)       & 0.47 & 1.69 & 0.30 & 0.24 \\
        Yelp Review          & 1.53 & 1.78 & 0.56 & 0.50 \\
        Multi-News           & 2.11 & 2.82 & 2.11 & 1.88 \\
        Reuters-21578        & 0.64 & 1.53 & 0.51 & 0.38 \\
        Yahoo! Answers        & 1.63 & 8.16 & 0.82 & 0.82 \\ 
    \end{tabular}
    \caption{Percentage of sentences entailed by their immediate context. $\cap$ is the intersection of sentences classified as entailment by both T5 and RoBERTa (RoB). $\cap^{+}$ is the percentage after manual filtering.}
    \label{tab:dataset-entailment}
\end{table}

\begin{exe}
    \ex \textbf{Yelp Review:}
    When he returned with it, he just placed it in front of me on the wet bar- no napkin/coaster, the beer was flat, and contained a FREAKING lemon. \entailed{Not an orange- a lemon.} \label{ex:emphasis}
\end{exe}

\noindent Beyond repetition, we also found examples where a weaker claim follows more specific premises:
\begin{exe}
    \ex \textbf{Yelp Review:}
    Frankly, I'm no oyster aficionado, but after comparing with other restaurant, it was pretty weak. In comparison to other oyster bars in the area, they were much to liquid-y. That is, they just didn't have enough substance on the whole and also, the taste wasn't really like seawater, it was more salt water than anything. \entailed{Fairly disappointed in the oysters.} \label{ex:oysters}
\end{exe}
In \Cref{ex:oysters}, the final sentences does not restate all the information from any previous sentence but rather makes a weaker claim that summarizes the review. In other cases, we find that the conclusion of logical arguments can behave similarly:
\begin{exe}
    \ex \textbf{Wikipedia:} All of the known sphenacodonts are carnivores except for certain therapsids. Glaucosaurus is plainly not a therapsid
    \ldots
    And it is just as plainly not a carnivore
    \ldots
    \entailed{So, it is very likely to be an edaphosaur.} \label{ex:explanation}
\end{exe}
With the world knowledge that a glaucosaurus must either be an edaphosaur or a sphenacodont,
the final sentence follows logically from the context. Thus, it seems the role of this entailed sentence is to make explicit the conclusion of a logical argument.

In summary, our corpus study reveals that more entailed text is uttered by humans than expected if humans were always avoiding redundancy, as Gricean speakers do.
There are many types of entailed text, including both repetition and instances where the entailed text is a summary or conclusion.
Next, we will consider how a Gricean speaker might be extended to account for this behavior.

\section{Towards Accounting for Redundancy} \label{sec:towards-accounting-for-redundancy}

We have found that, in practice, the flipped entailment test better detects entailment than the original one and that this trend is also supported by an oracle logistic regression analysis (\Cref{sec:base-test-results}).
Our corpus study (\Cref{sec:corpus-study}) pointed to a possible explanation: the original test relied on the fact that Gricean speakers always avoid redundancy, but real humans produce redundant text in certain contexts.
Quantitatively, the rate of contextually entailed sentences in natural corpora was higher
than we would expect if the corpus authors were Gricean speakers.
Qualitatively, specific examples suggested humans are redundant both to repeat important information and for the sake of explanation, i.e., they state entailed summaries or conclusions after a more detailed premise.
\emph{Prima facie}, such redundancy could lead to a flipped entailment test if entailed continuations, which are fully redundant, become more likely than other continuations. However, it is crucial to have a more concrete theory of \emph{why} speakers are redundant to evaluate this and ideally explain why the test direction varies across constructions.
We thus consider some possible angles to extend Gricean speakers to account for redundant speech acts and whether these extensions predict the flipped test.

\subsection{Redundancy via Noise Tolerance} \label{sec:redundancy-via-noise-tolerance}

Our corpus study showed that one type of redundancy in natural text unaccounted for by Gricean speakers is simple repetition.
For example, the speaker in \Cref{ex:emphasis} repeats the claim that the orange in their beer was not a lemon.
Gricean speakers are unlikely to generate such repetition, but they can be extended to do so by assuming there is noise in the communication channel, i.e., listeners may fail to interpret each sentence with some probability \citep{degen2019redundancy}.
In this setting, a rational speaker is incentivized to hedge the risk their listener might not understand important information by repeating it twice. We call such a speaker a \emph{noise-tolerant} speaker, which we formalize in \Cref{sec:noisy-speaker}.

Noise-tolerant speakers can better account for repetition than Gricean speakers, but, if we assume corpora are generated by noise-tolerant speakers, would it explain the flipped direction of the entailment test? The short answer seems to be no. In \S\ref{sec:noisy-speaker}, we derive an extension of the entailment test that ``cancels out'' noise tolerance by simply repeating the initial sentence in each term $n$ times:
\begin{align*}
    \hat E_p^n(x, y)
    &\triangleq \log p(x^ny) - \log p(x^n\$) \\
    &- \log p(y^{n+1}) + \log p(y^n\$) .
\end{align*}
As $n$ increases, this test approximates the original test for a Gricean speaker.
Thus, if the source of the flipped test was redundancy introduced by a speaker's goal of being noise-tolerant, this test should work unflipped. Instead, we find that the \emph{flipped} noise-tolerant test still detects entailment---in fact, better than the original flipped test. Post hoc analysis suggests the better performance may be due to the computational benefit of the additional tokens in the noise-tolerant test prompts.
In summary,
accounting for noise tolerance does not seem to explain why the test was flipped.

\subsection{Redundancy via Explanations} \label{sec:explanatory-speaker-main}

A theory of speakers based on noise tolerance does not seem to explain the flipped entailment test.
The noise-tolerant speaker accounts for repetition, but we also saw other kinds of redundancy in the data.
In particular, \Cref{ex:oysters,ex:explanation} show redundant sentences can occur at the end of an explanation or logical argument.
One account could be that
an initial explanation can dramatically lower the processing cost of a later conclusion, and that speakers consider this when selecting utterances.
This is not modeled by the Gricean speaker whose processing cost $c(y)$ is independent of the context $x$.
We thus reformulate the cost $c(y \mid x)$ as context-dependent.
The impact of $x$ on cost is measured by $\Delta(x, y) \triangleq c(y) - c(y \mid x)$:
a large $\Delta(x, y)$ indicates a concise but helpful explanation $x$ before conclusion $y$.
If $\Delta(x, y)$ is large enough, the speaker will prefer to say $xy$ as opposed to just $y$.

\paragraph{Flipped Test.}
Let $E(x, y)$ be the desired semantic value of the entailment test.
With an explanatory speaker, the test score becomes
(see \Cref{sec:explanatory-speaker}):
\begin{equation*}
    \hat E_p(x, y) = E(x, y) + \Delta(x, y) - \Delta(y, y) .
\end{equation*}
If we assume $\Delta(x, y)$ dominates $E(x, y)$, the test score can \emph{increase} when $x$ entails $y$ because $x$ will often explain $y$.
This might explain the flipped test pattern.
However, to be more complete, this account should be more precise about what factors influence $c(y \mid x)$ and predict why the original test outperformed the flipped test in some cases.

\subsection{Discussion}

Since we found that the entailment test was flipped in practice and that there are cases where humans are more redundant than Gricean speakers, we explored extensions to the Gricean speaker that could more accurately account for human redundancy and thus better explain the flipped test. 
We first considered a test that accounts for redundancy due to noise tolerance, finding that this likely could not explain the flipped test.
Motivated by \Cref{sec:corpus-study}, we then turned to explanations as another source of human redundancy and showed how accounting for explanations might predict the flipped test.
\footnote{Another reason speakers may be redundant, which we have not considered, is to trigger the listener to reanalyze the question under discussion. E.g., \Cref{ex:emphasis} may prompt the listener to infer the speaker's goal is to express frustration rather than convey the facts of their order.}
We take this as encouraging evidence for pursuing pragmatic theories that explicitly account for explanations.

Stepping back, we have been able to use LMs as a source of data about sentence co-occurrences to test pragmatics theories and motivate alternatives, in the spirit of \citet{harris1954distributional}'s idea that corpus data should be the empirical foundation of linguistic theory.
A fundamental problem with using corpus data has been data sparsity,
but LMs can alleviate this by letting us interpolate the likelihood of arbitrary sentences.
We believe this could be a promising paradigm for future research in computational pragmatics to complement human subject experiments.

\section{Conclusion}

Our results show that sentence co-occurrence probabilities computed by LMs can predict entailment relationships, with a stronger effect for better LMs. This suggests these LMs are implicitly modeling semantic properties of text to some extent in order to predict the next token, in line with \citet{harris1954distributional}'s proposal that sentence co-occurrences can serve as data for building a theory of semantics.
However, the best empirical test for entailment we found was flipped compared to \citet{merrill-etal-2022-entailment}'s theoretical test. This suggests a more nuanced theory of pragmatics beyond Gricean speakers is needed to explain how entailment relationships are reflected in sentence co-occurrences. Our corpus study revealed that humans in corpora produce more contextually entailed sentences than idealized Gricean speakers, suggesting pragmatic theories that better handle redundancy might explain our findings.

We took a first step by considering how to model redundancy due to noise tolerance and explanation, but the job is far from done.
Rather, our findings call for future work that more completely accounts for the pragmatics of redundancy, especially concerning explanations. This can both advance linguistic theory and serve as a foundation for understanding how meaning can be inferred from a corpus, as well as
as the potential limits of distributional semantics and LMs.

\section*{Limitations}
Regarding the theoretical foundations for the entailment test, \citet{merrill-etal-2022-entailment} indicate in an erratum that the entailment test may have false positives for rare sentences pairs that are nearly contradictory.
Further, the theory may be less applicable to LMs that have undergone an alignment process like RLHF.
Overall, these qualifications to the test theory increase the value of our empirical study of whether the test works in practice.

Regarding our analysis of our results, we have assumed that the flipped entailment test pattern reflects differences between Gricean speakers and human speakers in corpora, but it, in principle, systematic estimation errors by LMs could explain the flipped entailment test pattern independent of the distribution of strings in the training corpus.

\section*{Acknowledgements}
We thank Emmanuel Chemla, Noah Goodman, Sophie Hao, He He, Nitish Joshi, Alisa Liu, Ashish Sabharwal, and Benjamin Spector for insightful discussions and comments. This project benefited from NYU HPC resources and expertise. WM was supported by an NSF graduate research fellowship, AI2, and Two Sigma. ZW and YK were partially supported by funds from  MIT-IBM Watson AI and Amazon grants.

\bibliography{references}
\clearpage
\appendix
\onecolumn
\section{Test Derivation for Gricean Speakers} \label{sec:test-details}

As shown by \citet{merrill-etal-2022-entailment}, the entailment test score $\hat E_p(x, y)$ score defined in terms of co-occurrence log-probabilities is equivalent to the following semantic quantity:
\begin{equation*}
    E(x, y) \triangleq \log \frac{\EX_w [\exp(i_\ell(xy \mid w)) g(x, w)]}{\EX_w [\exp(i_\ell(x \mid w)) g(x, w)]} ,
\end{equation*}
where $g(x, w)$ captures the normalizing factor from the speaker \citep[cf.][]{merrill-etal-2022-entailment}.
\begin{proposition}[\citealp{merrill-etal-2022-entailment}]
Let $p$ be a Gricean speaker.
Then, for any $x, y$, $\hat E_p(x, y) = E(x, y)$. 
\end{proposition}

\begin{proof}
    We recount an abbreviated version of the proof from \citet[Appendices C and H]{merrill-etal-2022-entailment}. We use the fact that, for any $x, y$,
    \begin{equation*}
        \log p(xy) - \log p(x\$) = E(x, y) - c(xy) + c(x\$).
    \end{equation*}
    Applying this property to both sides of $\hat E_p(x, y)$ yields
    \begin{align*}
        \hat E_p(x, y)
        &= \log p(xy) - \log p(x\$) - \log p(yy) + \log p(y\$) \\
        &= E(x, y) - c(xy) + c(x\$) - \cancel{E(y, y)} + c(yy) - c(y\$) \\
        &= E(x, y) + \cancel{c(xy^2\$)} - \cancel{c(xy^2\$)} .
    \end{align*}
    We conclude that $\hat E_p(x, y) = E(x, y)$. 
\end{proof}

Crucially, $E(x, y)$ is closely related to entailment.
If $x$ entails $y$, then $y$ conveys no information after $x$, so $E(x, y) = 0$. On the other hand, if $E(x, y) = 0$, then it must either be that a) $x$ entails $y$ or b) $y$ nearly contradicts $x$, meaning the probability that $x, y$ are consistent is small \citep[][Erratum]{merrill-etal-2022-entailment}.
Assuming near contradiction is unlikely, the entailment test (since it computes $E$) is then effectively a test for entailment defined purely in terms of sentence co-occurrence probabilities.

\section{Noise-Tolerant Speakers} \label{sec:noisy-speaker}

We now formalize a model of noise-tolerant speakers that can account for repetition (\Cref{ex:emphasis}).
Our speaker is inspired by \citet{degen2019redundancy}'s speaker designed to account for overredundant referring expressions but extends better to multiple sentences.
We assume each sentence $x$ has some probability $\epsilon_x$ of not being interpreted. When anticipating the information a listener gains from a text, a speaker marginalizes over the potential interpretations the listener might form by failing to interpret different sentences:
\begin{equation*}
    p(z \mid w) \propto \EX_{e} [ \exp(i_\ell(e \mid w)) ] \exp(-c(z)) ,
\end{equation*}
where $e$ is a set of indices for sentences in $z$ that are full comprehended.
Formally, $e$ is a subset of $z$'s indices representing a subsequence.
Note that $i_\ell(e \mid w)$ is defined in the natural way: it is the information a listener would get from just the sentences of $z$ activated in $e$ and not the other ones. This implicitly depends on $z$.
The distribution of $e$ is determined by $\epsilon$'s for each sentence in $z$:
\begin{equation*}
    p(e \mid w, z) = \prod_{t=1}^n \begin{cases}
        1 - \epsilon_{z_t} & \textrm{if} \; t \in e \\
        \epsilon_{z_t} & \textrm{otherwise.}
    \end{cases}
\end{equation*}

\subsection{Theoretical Result}
The original entailment test does not hold for noise-tolerant speakers, but a straightforward extension does. For any $n \geq 1$, we define the extended test as
\begin{equation}
\begin{split} \label{eq:repetition}
    \hat E_p^n(x, y)
    &\triangleq \log p(x^ny) - \log p(x^n\$) \\
    &- \log p(y^{n+1}) + \log p(y^n\$) .
\end{split}
\end{equation}
This extended test (with $p$ as a noise-tolerant speaker) approximates the original test for a Gricean speaker, with error vanishing exponentially in $n$:
\begin{proposition}
    Let $p$ be a noise-tolerant speaker.
    As $n$ increases, $\hat E_p^n(x, y)$ converges to $E(x, y)$ with error vanishing exponentially in $n$.
\end{proposition}
\begin{proof}
The idea is that, unlike a Gricean speaker, a noise-tolerant speaker will produce $p(ab)$ to account for the chance that $a$ was not interpreted. If $a$ repeats several times, the chance $a$ was not interpreted goes to $0$.

In order to show that the original test fails with this speaker and work out an alternative, we first work out some basic properties of this speaker's utility.
Let $\dot I(z \mid w) \triangleq \EX_{e} [ i_\ell(e \mid w) ]$ be the expected utility of $z$.
We can first characterize the utility of a $2$-gram $xy$ under the noisy-channel speaker:
\begin{align*}
    \dot I(xy \mid w)
    &= \epsilon_x \epsilon_y \cdot 0 + (1 - \epsilon_x) \epsilon_y i_\ell(x \mid w) + \epsilon_x (1 - \epsilon_y) i_\ell(y \mid w) + (1 - \epsilon_x)(1 - \epsilon_y) i_\ell(xy \mid w) \\
    &= (1 - \epsilon_x) \epsilon_y i_\ell(x \mid w) + \epsilon_x (1 - \epsilon_y) i_\ell(y \mid w) + (1 - \epsilon_x)(1 - \epsilon_y) i_\ell(xy \mid w) .
\end{align*}
We can apply this to get the expected utility of the utterances $xx$ and $x\$$ under the noisy-channel speaker:
\begin{align*}
    \dot I(xx \mid w)
    &= \epsilon_x^2 \cdot 0 + 2 (1 - \epsilon_x) \epsilon_x i_\ell(x \mid w) + (1 - \epsilon_x^2) i_\ell(x \mid w) \\
    &= (1 - \epsilon_x^2) i_\ell(x \mid w) \\
    \dot I(x\$ \mid w)
    &= (1 - \epsilon_x) \epsilon_\$ i_\ell(x \mid w) + (1 - \epsilon_x)(1 - \epsilon_\$) i_\ell(x \mid w) \\
    &= (1 - \epsilon_x) i_\ell(x \mid w) .
\end{align*}

We can now see that the original test does not work under a noise-tolerant speaker.
The original entailment theorem worked by checking $i_\ell(y \mid x, s) = i_\ell(x \mid x, s)$ to see whether $y$ is informative after $x$.
Naively applying the original entailment test with a noise-tolerant speaker, however, will use $\dot I$ in place of $i_\ell$.
We can see that this does not represent the same quantity if $\epsilon_x, \epsilon_y$ are non-negligible:
\begin{align*}
    \dot I(x \mid x, w) &= \epsilon_x (1 - \epsilon_x) i_\ell(x \mid w) \\
    \dot I(y \mid x, w) &= \epsilon_x (1 - \epsilon_y) i_\ell(y \mid w) + (1 - \epsilon_x)(1 - \epsilon_y) i_\ell(y \mid x, w) .
\end{align*}

However, for the new test, we find the following:
\begin{align*}
    \dot I(x \mid x^n, w)
    &= \epsilon_x^n (1 - \epsilon_x) i_\ell(x \mid w)
    \approx 0 \\
    \dot I(y \mid x^n, w)
    &= \epsilon_x^n (1 - \epsilon_y) i_\ell(y \mid w) + (1 - \epsilon_x^n)(1 - \epsilon_y) i_\ell(y \mid x, w)
    \approx (1 - \epsilon_y) i_\ell(y \mid x, w) .
\end{align*}
For large $n$, this overcomes the $\epsilon$'s since it means the test checks whether $i_\ell(y \mid x^n, w)$ is nonzero for all $w$ (assuming $\epsilon_y < 1$, i.e., a human can possibly evaluate $y$).
\end{proof}

\subsection{Empirical Results}

\begin{figure*}[t!]
    \centering
    \includegraphics[width=0.48\textwidth]{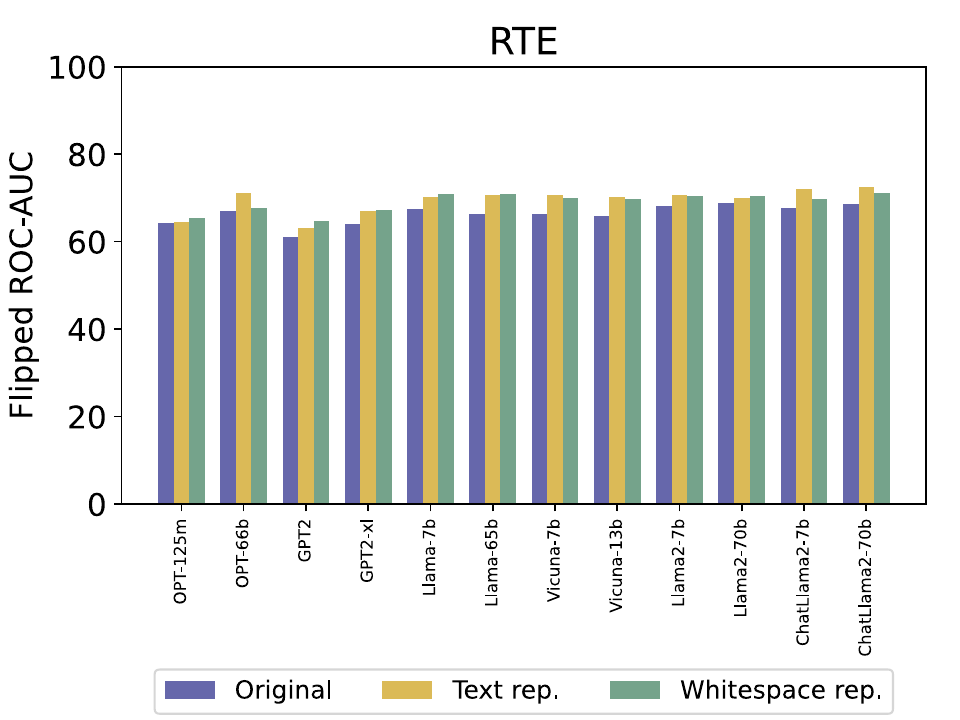}
    \includegraphics[width=0.48\textwidth]{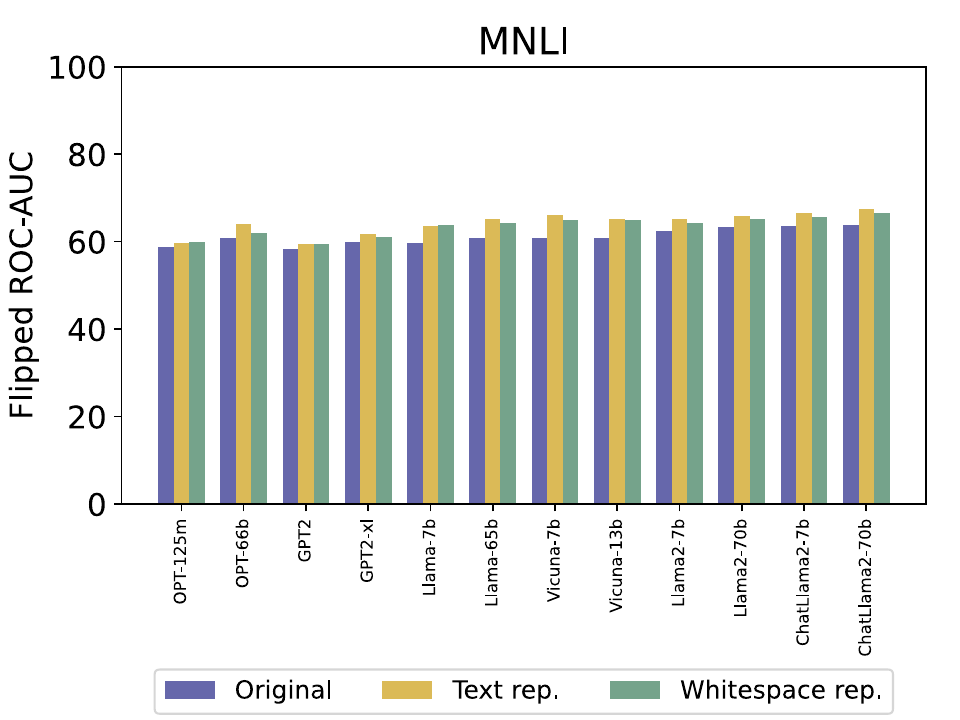}
    \caption{Performance of noise-tolerant (\S\ref{sec:noisy-speaker}) vs.~original test on RTE training set and MNLI matched validation set.}
    \label{fig:repetition}
\end{figure*}

\begin{figure}[t!]
    \centering
    \includegraphics[width=0.48\textwidth]{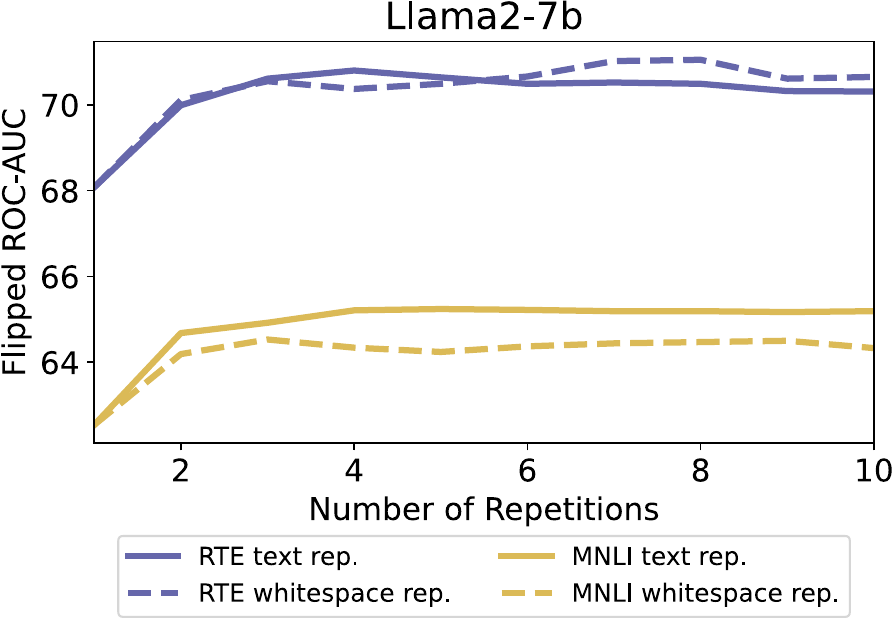}
    \caption{Performance of the noise-tolerance test with different numbers of repetitions (values of $n$ in \Cref{eq:repetition}). The original test is $n=1$.}
    \label{fig:num-repetitions}
\end{figure}

We compare the noise-tolerant test with $n=5$ repetitions against the original test, using the RTE training set and the MNLI matched validation set.\footnote{Due to the repetitions multiplicatively increase sequence length, running this test on the MNLI training set, as we do in the other experiments, was not feasible for us.} As shown in \Cref{fig:repetition}, the flipped noise-tolerant test consistently detects entailment better than the original flipped test. However, the fact that the test still works better flipped is just as unexpected with the noise-tolerant test as with the original test.

We were thus skeptical whether the boost in performance from the noise-tolerant test was due to more realistic speaker assumptions and aimed to access whether there could be a confounding explanation.
In particular, in addition to accounting for ways speakers can be redundant, the noise-tolerant grants the LM additional tokens and thus more steps of computation, which could enable more closely approximating each log-likelihood \citep{goyal2023think}. To control for this, we introduce a ``pause token'' test where, for each term $\log p(ab)$, spaces are inserted between $a$ and $b$ to add the same number of tokens that would be added by replacing $a$ with $a^n$.\footnote{The tokenizer for Llama models treats 16 consecutive whitespaces as a single token. We hence insert 16 times more whitespaces for Llama-based models to control for the token count.} Assuming spaces carry no semantics, the pause token test should measure the same quantity as the original entailment test, but with more compute than the noise-tolerant test.

As shown in \Cref{fig:repetition}, the pause token test outperforms the original test, suggesting the computational benefit of additional tokens may explain the test improvement. For many datasets, the pause token test performs slightly worse than the noise-tolerant test, but because the absolute difference is small and not consistent, we do not take this as evidence that the noise-tolerant test provides a benefit beyond more tokens of computation.
Further, \Cref{fig:num-repetitions} shows that increasing the number of repetitions yields roughly monotonic but diminishing returns, as might be expected for a computational resource.
Overall, we conclude the stronger performance of the noise-tolerant test likely reflects the greater computational power of padding tokens and not better assumptions about human speakers.

\section{Explanatory Speakers} \label{sec:explanatory-speaker}

The only change we make to the speaker to support explanations is generalizing the cost $c(y \mid x)$ to depend on the prior context.
We assume that $c(\$ \mid z) = c(\$)$ for all $z$.

\begin{proposition}
    Let $p$ be an explanatory speaker. Then, for any $x, y$,
    \begin{equation*}
        \hat E_p(x, y) = E(x, y) + \Delta(x, y) - \Delta(y, y) .
    \end{equation*}
\end{proposition}

\begin{proof}
By definition,
\begin{align*}
    \hat E_p(x, y)
    &= E(x, y) - c(xy) + c(x\$) + c(yy) - c(y\$) \\
    &= E(x, y) - c(y \mid x) + c(\$ \mid x) + c(y \mid y) - c(\$ \mid y) .
\end{align*}
These cost terms do not all cancel out (as for Gricean speakers). Instead, we get
\begin{align*}
    \hat E_p(x, y)
    &= E(x, y) - c(y \mid x) + \cancel{c(\$)} + c(y \mid y) - \cancel{c(\$)} \\
    &= E(x, y) - c(y \mid x) + c(y \mid y) + c(y) - c(y) \\
    &= E(x, y) + \Delta(x, y) - \Delta(y, y) .
\end{align*}
\end{proof}

\subsection{Further Details and Experiments} \label{sec:explanations-details}

To be convincing, the explanatory speaker account should ideally explain why the original test worked better than the flipped test for some targeted cases like logical connectives and numbers (\Cref{fig:aucroc-histogram}).
The connectives could possibly be explained by the fact that the connectives hypotheses introduced new entities that did not occur in the premise (cf.~\Cref{ex:connectives}).
Because these entities do not exist in the discourse, it would be infeasible for a listener to reason about whether they are entailed in advance, making semantic priming unlikely.
We would thus expect $\Delta(x, y)$ and the test to better match $E(x, y)$ in this case.

\begin{figure*}[!t]
    \includegraphics[width=.48\columnwidth]{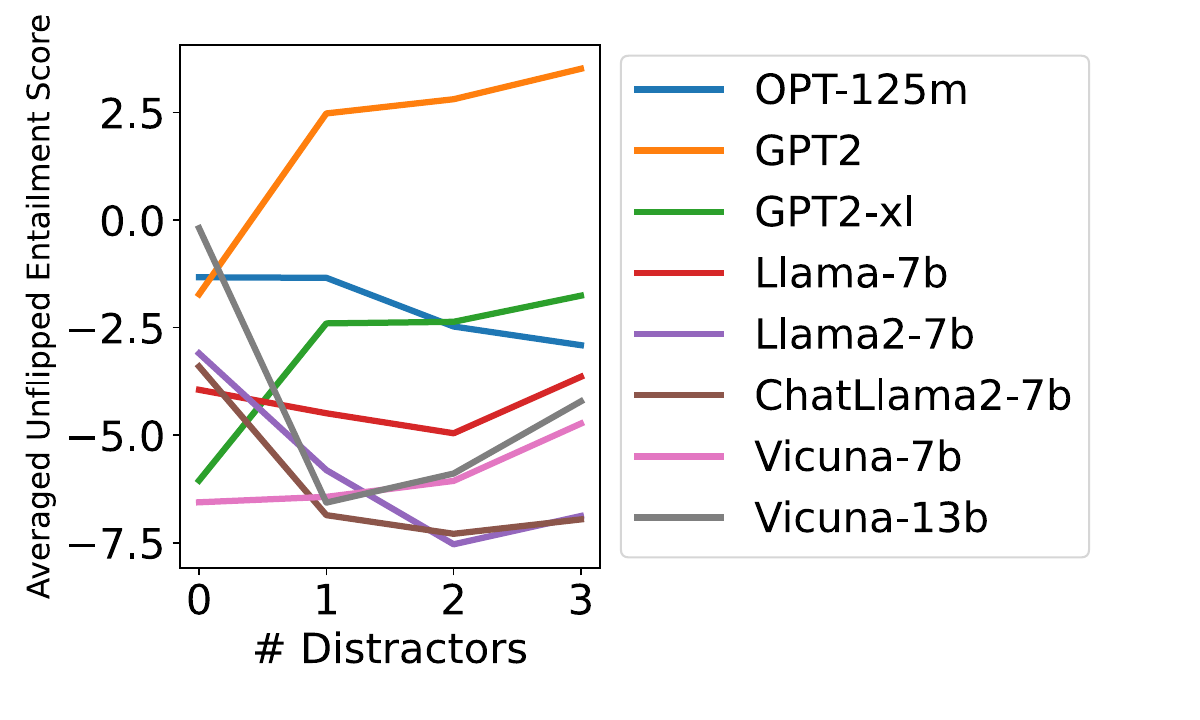}
    \includegraphics[width=.48\columnwidth]{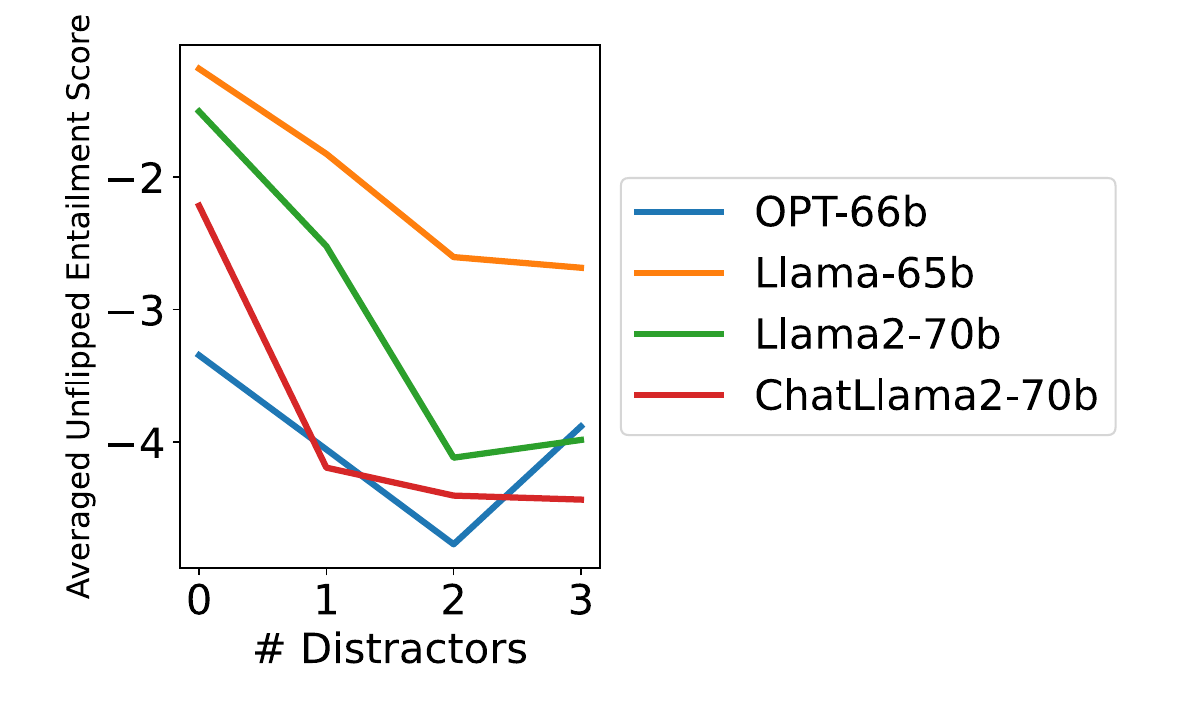}
    \caption{Unflipped entailment test score as a function of the number of distractors in the premise, with $<$65b models (left) and $\ge$65b models (right), for the RTE dataset.}
    \label{fig:distractors-70b}
\end{figure*}

The semantic priming account predicts that, for entailed pairs, the test score should reflect how much $x$ semantically primes $y$. Assuming adding distractors to the premise reduces semantic priming, it thus predicts that the entailment score should decrease as more distractors are added to the premise. We test this by generating entailment pairs with distractors in the premise like the following:
\begin{exe}
    \ex Olivia lives in Paris. James lives in Tokyo. \entailed{Olivia lives in France.}
\end{exe}
As shown in \Cref{fig:distractors-70b}, this pattern holds for all $\sim$70b LMs we considered, although the results for LMs of smaller scales are more inconsistent.
We take this as weak evidence that the speakers LMs are modeling (i.e., humans) may be accounting for the reduction in processing time that an explanation can provide.

\section{Manual Inspection of Entailment Classified by Models} \label{sec:entailment-manual-inspection}

The following are examples of premise-hypothesis pairs which were marked as entailment by both T5~\citep{honovich-etal-2022-true-evaluating} and RoBERTa \citep{liu2019roberta}.
Through manual inspection, however, we find that they were in fact incorrectly classified as such.
We include a comprehensive list of those cases as well as reasoning as to why we believe the pair is not entailment.

\begin{exe}
    \ex \textbf{Multi-News:}
    The man survived the fall and the waters.After he was rescued, he noted that a "burning platform" caused a radical change in his behaviour.We too, are standing on a "burning platform," and we must decide how we are going to change our behaviour.Over the past few months, I've shared with you what I've heard from our shareholders, operators, developers, suppliers and from you.Today, I'm going to share what I've learned and what I have come to believe. \entailed{I have learned that we are standing on a burning platform.}
\end{exe}
The premise does not contain information regarding the fact that the narrator had "learned [they] are standing on a burning platform".

\begin{exe}
    \ex \textbf{Books3:}
    That's where you're wrong.I only have negatives.Minus wishes.""E, what are you going on about?"I asked gently, leaning in and wincing as my shirt caught on the dressing. \entailed{"I only know what I don't want.}
\end{exe}
The "I only have negatives" in the premise can be interpreted as the narrator only having things that they don't want.
This is in contrast to the knowledge aspect which is brought up in the hypothesis.

\begin{exe}
    \ex \textbf{Book3:}
    I glance over my shoulder.Liam moves fast too, throwing himself at the hill to catch me.It's fine, I can outrun him over distance.All I need is a head start.So I push myself, stumbling on the dry churned-up turf.	\entailed{Behind me, Liam speeds up.}
\end{exe}
The premise indicates that "Liam moves fast...to catch me".
It does entail that he "speeds up", which is in the hypothesis.

\begin{exe}
    \ex \textbf{Wikipedia:}
    Henry admits he doesn't dance, and encourages Minnie to dance with Sidney.Henry thinks Minnie must find life with him dull, and resolves to learn to dance.He keeps this secret from her in order to surprise her on her birthday.He takes private dancing lessons, instructed by Madame Gavarni and her niece.Minnie seems to grow distant. \entailed{Henry thinks she is bored, and looks forward to surprising her with dancing.}
\end{exe}
While Henry thinking Minnie is bored and planning on surprising her with dancing, him "[looking] forward to [it]" is new information presented in the hypothesis not in the premise.

\begin{exe}
    \ex \textbf{Wikipedia:}
    "Established in 1923, it has a membership of around 230,000 and is open to past and present members of the UK Civil Service and public sector plus organisations that were formerly part of the British Civil Service, for instance Royal Mail and BT.Relatives of existing members may also join.History
    Boundless by CSMA is a mutual organisation.It was founded as the Civil Service Motoring Association in 1923 by Frank Vernon Edwards, an executive officer in the Ministry of Labour who had an interest in motorcycle trials.CSMA Club was designed to be a small motorsport organisation of around 300 members, but by 1930 the membership was over 5,000." \entailed{The membership currently stands over 230,000.}
\end{exe}
The premise does not indicate that there are more than 230,000 members, which means that the hypothesis is adding additional information not contained in the premise.

\begin{exe}
    \ex \textbf{Reuters-21578:}
    "A spokeswoman for the EC Commission said the detailed 25-page report of alleged malpractices was in response to a similar document issued by U.S. Administration officials in November, and updated a previous EC list.EC External Trade Relations Commissioner Willy De Clercq said its object was to show such actions were not solely taken by trading partners of the U.S. And that "the U.S.Were not innocents in the matter."The report covers the entire field of EC-U.S. Commercial relations and lists more than 30 obstacles ranging from tariff measures, import quotas, customs duties, anti-dumping procedures, fiscal barriers and export subsidies.The Commission said not all the barriers mentioned were necessarily inconsistent with U.S. International obligations, and emphasised many of them could be removed at upcoming international trade talks." \entailed{The purpose of the report is to make clear that trade practices which impede exports are not a unique problem only faced by U.S.}
\end{exe}
The premise describes something different from the hypothesis in that the objective of the report was to show "such actions were not solely taken by trading partners of the U.S." but also participated in by the U.S. itself.

\begin{exe}
    \ex \textbf{Yelp Review:}
    "While it looks decent on the outside and the inside, the food and service were simply terrible.Chicken was very watered down, the salsa was flavorless, and the service make a fast food chain look really good.Just a poor, poor experience at this location overall.If this was the only El Cancun in Charlotte, I would feel the same way many posters do and just never come back. Luckily for me, I live in Rock Hill." \entailed{There's an El Cancun here.}
\end{exe}
The hypothesis introduces new information about another El Cancun location being where the speaker is, which is not present in the premise.

\section{Manual Classification of Entailment Categories} \label{sec:entailment-categories}

Based on the filtered manual results described by $\cap^{+}$ in \Cref{tab:dataset-entailment}, we manually classify results into three categories: explanation, repetition, and other.
We define an entailment pair as ``repetition'' if there is a single span $s$ in the premise such that $s$ entails the hypothesis and vice versa. We define ``explanation'' as any pair where it is clear that no span in the context entails the hypothesis and is also entailed by it.
Finally, ``other'' represents cases where we cannot clearly determine these conditions.
Results of these classifications across datasets are shown in \Cref{fig:entailment-categories}. 
We acknowledge that there is some unavoidable subjectivity involved with the manual filtering and classifications, but we think that our manual classification is somewhat instructive despite this limitation.

\begin{figure}[t!]
    \centering
    \includegraphics[width=0.48\textwidth]{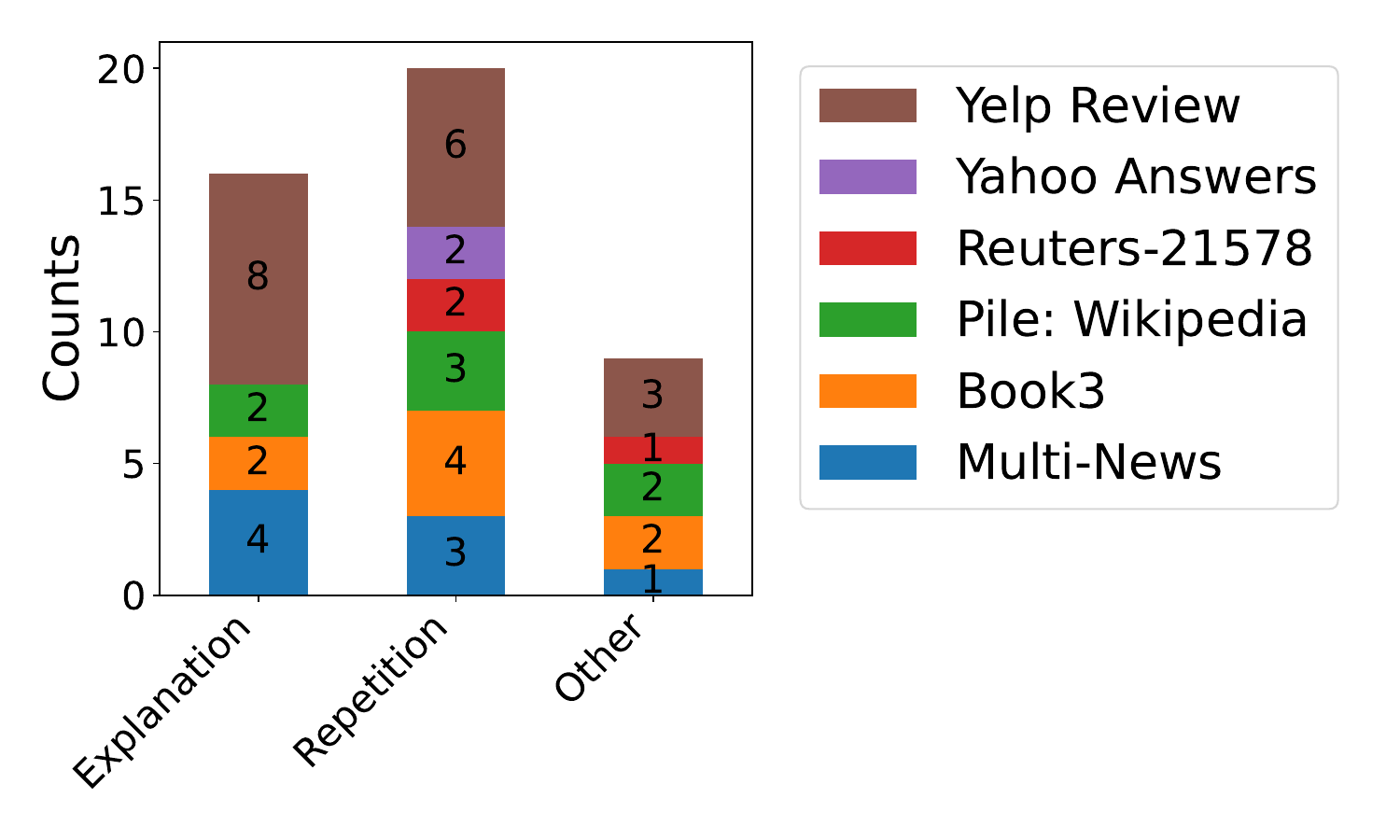}
    \caption{Frequency of occurrences of entailment categories across data sources.}
    \label{fig:entailment-categories}
\end{figure}

\section{Learned Entailment Test for More Datasets} \label{sec:learned-entailment-test-for-more-datasets}

\begin{figure*}[t!]
    \centering
    \includegraphics[width=\textwidth]{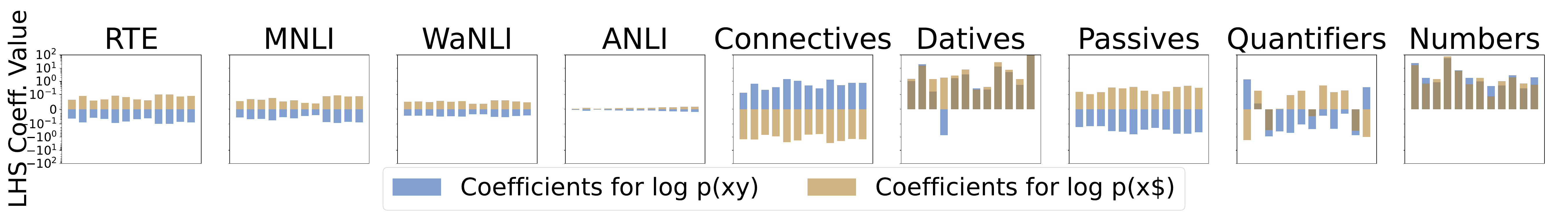}
    \includegraphics[width=\textwidth]{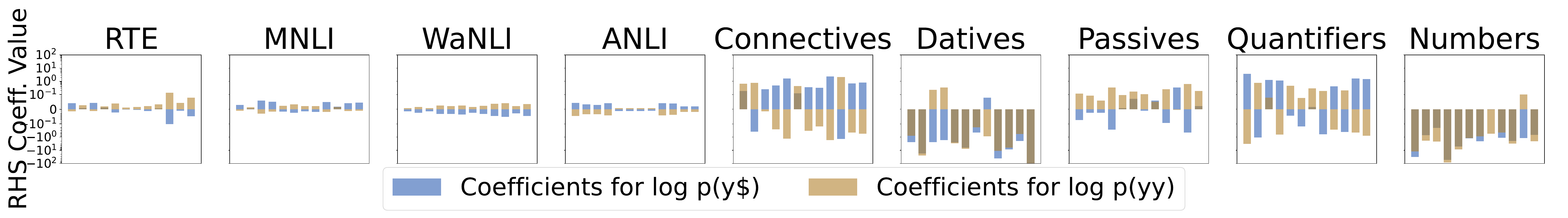}
    \caption{Learned logistic regression coefficients for the log-probability features for the broad-coverage datasets. Each bar represents one LM. For ease of visualization, $y$-axis is in log scale, except in $[-0.1, 0.1]$ where it is linear; and it is capped at $[-100, 100]$, requiring truncation in a few cases.
    }
    \label{fig:all-learned-coefficients-paired}
\end{figure*}

In \Cref{fig:all-learned-coefficients-paired}, we show the coefficients for the learned entailment test (\S\ref{sec:learning-a-distributional-entailment-test}). However, we note a caveat for the targeted evaluation datasets: because they are manually curated, there are simple dataset artifacts that can be used to distinguish between the two classes (for example, some types of hypotheses only exist for entailment instances). When we learned a classifier, such artifacts could be exploited (and we do see that they are exploited in practice). We thus highlight that the interpretation of the relevant coefficients are not straightforward.

\section{Synthetic Data for Targeted Evaluation} \label{sec:targeted-eval}

The GLUE diagnostics \citep{wang-etal-2018-glue} are not a public dataset; hence, we make our synthetic targeted evaluation data based on the GLUE design principles. We create synthetic data following the following templates, where the names and base propositions vary according to a hard-coded list.

\paragraph{Connectives.} The premise $p(a)$ entails $p(a \vee b)$ but not $p(a \wedge b)$:
\begin{exe}
    \ex I saw James. \label{ex:connectives}
    \begin{xlist}
        \ex I saw James or Olivia. \; \cmark
        \ex I saw James and Olivia. \; \xmark
    \end{xlist}
\end{exe}

\paragraph{Quantifiers.} For a non-empty domain, $\textit{all} \, a \, p(a)$ entails $\textit{some} \, a \, p(a)$ but not $\textit{no} \, a \, p(a)$:
\begin{exe}
    \ex All of the crops failed.
    \begin{xlist}
        \ex Some of the crops failed. \; \cmark
        \ex None of the crops failed. \; \xmark 
    \end{xlist}
\end{exe}

\paragraph{Numbers.} Similarly, \emph{at least two} entails \emph{at least one} but not \emph{at least three}:
\begin{exe}
    \ex At least two of the crops failed.
    \begin{xlist}
        \ex At least one of the crops failed. \cmark
        \ex At least three of the crops failed. \xmark
    \end{xlist}
\end{exe}

\paragraph{Passivization.} Given a premise with a transitive verb, the reduced passive with the original object as the subject is entailed, but the reduced passive with the original subject as the subject is not:
\begin{exe}
    \ex Olivia saw Mia.
    \begin{xlist}
        \ex Mia was seen. \quad \cmark
        \ex Olivia was seen. \quad \xmark
    \end{xlist}
\end{exe}

\paragraph{Datives.} Given a sentence with a direct object and an optional indirect object, the sentence with the indirect object removed is entailed, but the sentence with the direct object is not:
\begin{exe}
    \ex Liam baked Noah a cake.
    \begin{xlist}
        \ex Liam baked a cake. \; \cmark
        \ex Liam baked Noah. \; \xmark
    \end{xlist}
\end{exe}

\section{Language Models We Used} \label{sec:lms-we-used}

We test a variety of LM families, and for each, we use the smallest and largest public-available variant. Specifically, we use GPT-2 small (117M parameters) and XL (1.5B), OPT 125M and 66B, Llama-1 7B and 65B, Vicuna 7B and 13B, Llama-2 7B and 70B, and ChatLlama-2 7B and 70B.

\section{Dataset Stastics} \label{sec:dataset-stats}

\begin{table}[t!]
    \centering
    \begin{tabular}{lc}
        \toprule
        Dataset & \# Instances \\
        \midrule
        RTE-train & \phantom{00}2,490 \\
        MNLI-train & 392,702 \\
        MNLI-validation-matched & \phantom{00}9,815 \\
        WaNLI-train & 102,885 \\
        ANLI-train & 100,459 \\
        Connectives & \phantom{00}1,800 \\
        Quantifiers & \phantom{000,}780 \\
        Numbers & \phantom{000,}260 \\
        Passives & \phantom{00}2,160 \\
        Datives & \phantom{000,}720 \\
        \bottomrule
    \end{tabular}
    \caption{\label{tab:dataset-stats}
    The number of instances for each dataset we use.
    }
\end{table}

We report dataset statistics in \Cref{tab:dataset-stats}.

\section{Results Excluding Contradiction Instances}

In Figures~\ref{fig:aucroc-histogram-noc} to \ref{fig:all-learned-coefficients-paired-noc}, we report results for all of our analyses but excluding contradiction instances.
The motivation for this is to specifically target the relationship in scores between neutral and entailment pairs, which is the case where the empirical direction of the test is flipped compared to our theoretical expectation.

\begin{figure}[t!]
    \centering
    \includegraphics[width=0.44\textwidth]{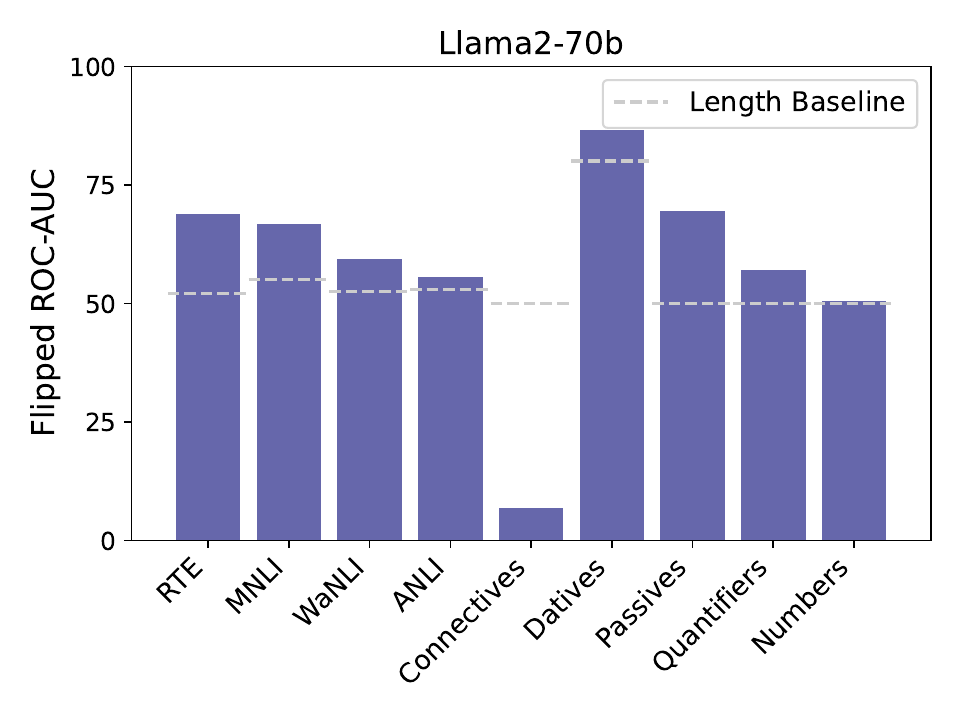}
    \vspace{-2mm}
    \caption{Flipped AUC-ROC scores of the flipped entailment test across datasets using Llama2-70b probabilities. All contradiction instances are excluded.}
    \vspace{-2mm}
    \label{fig:aucroc-histogram-noc}
\end{figure}

\begin{figure*}[t!]
    \centering
    \includegraphics[width=\textwidth]{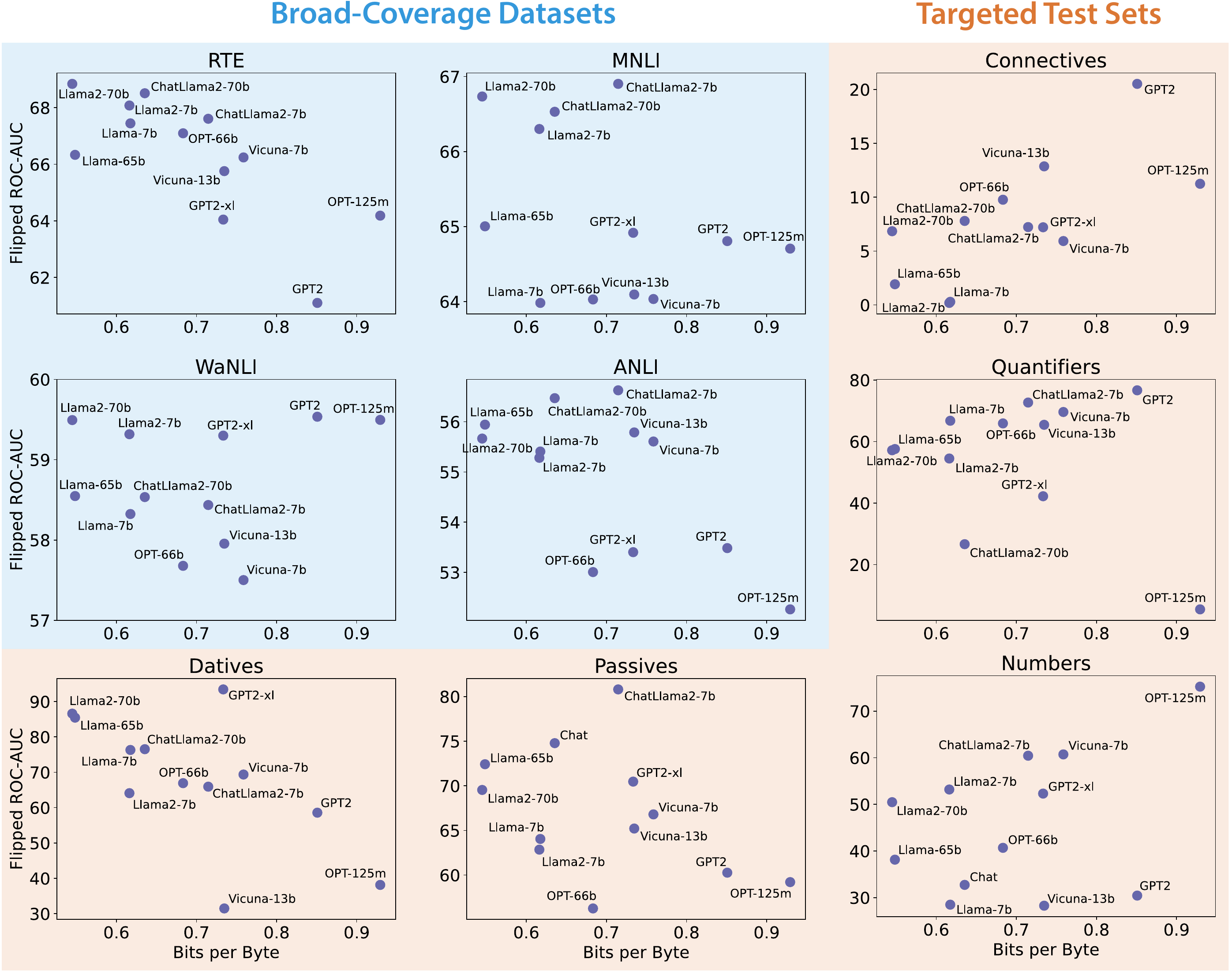}
    \caption{C4 validation bits per byte vs.~flipped AUC-ROC score for all models on broad-coverage
    and targeted datasets. Note that the scale of the $y$-axis differs for each subplot. 
     All contradiction instances are excluded.}
    \label{fig:aucroc-vs-perplexity-noc}
\end{figure*}

\begin{figure}[!t]
    \centering
    \includegraphics[width=0.48\textwidth]{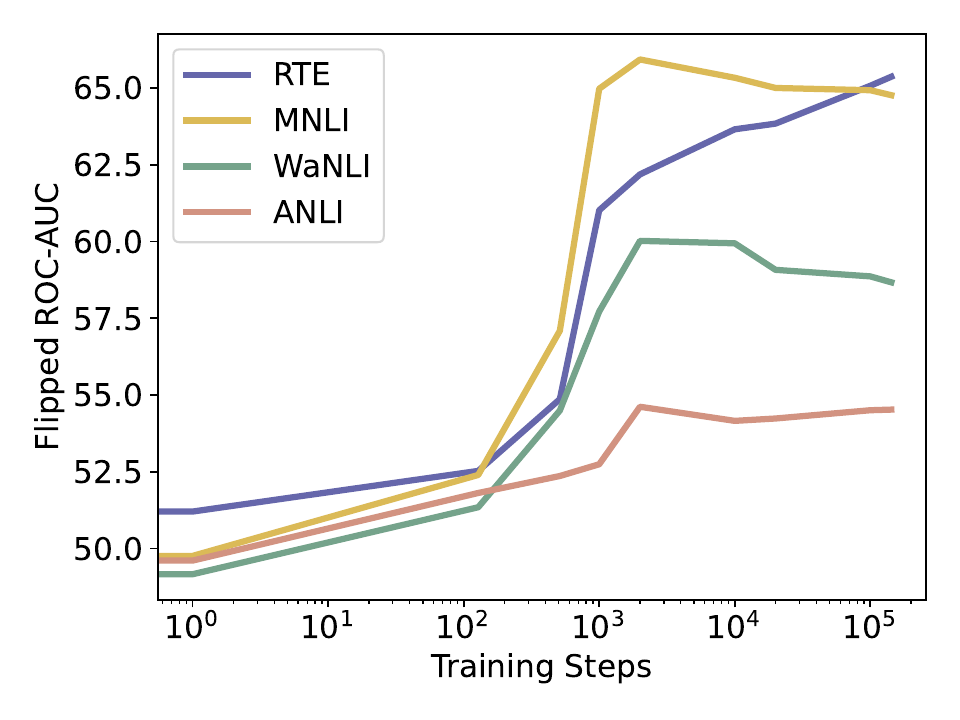}
    \caption{Flipped ROC-AUC of entailment score across Pythia-12b checkpoints. Each step is around 2M tokens.  All contradiction instances are excluded.}
    \label{fig:pythia-noc}
\end{figure}

\begin{figure}[t!]
    \centering
    \includegraphics[width=0.48\textwidth]{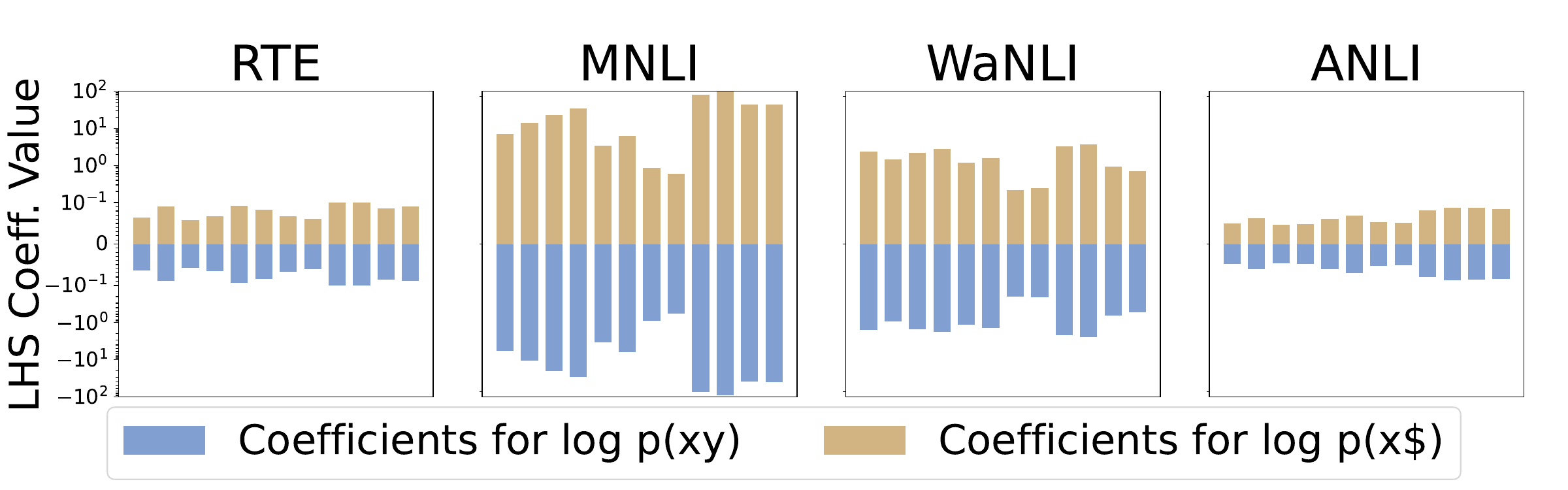}
    \includegraphics[width=0.48\textwidth]{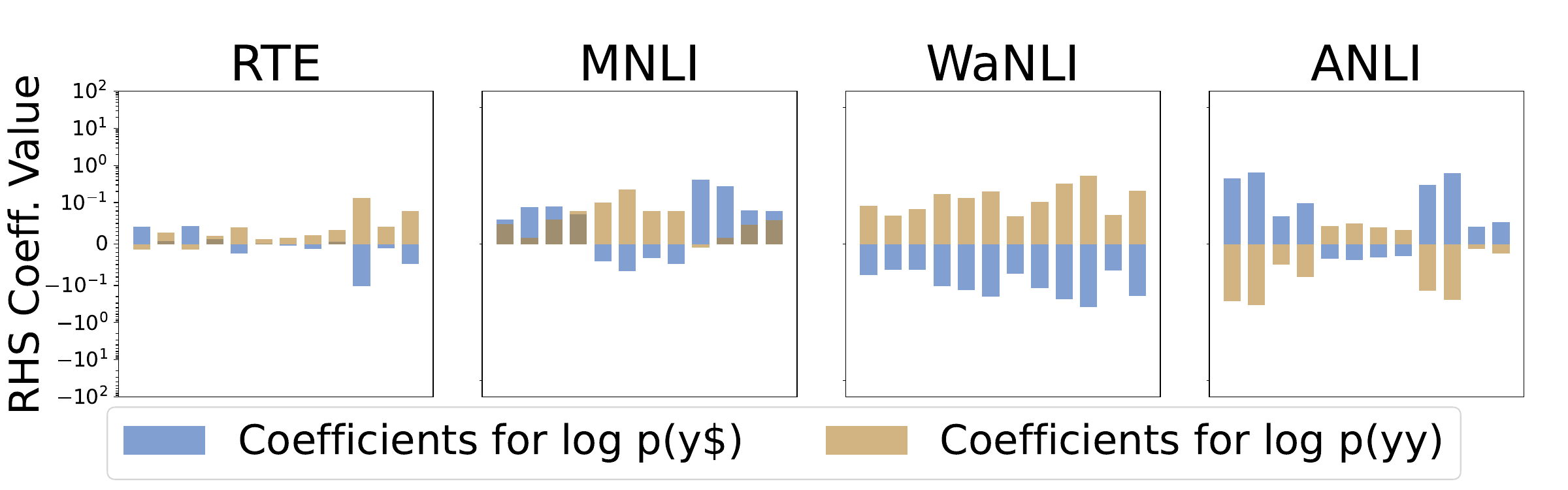}
    \caption{Learned logistic regression coefficients for the log-prob features for the broad-coverage datasets. Each bar represents one LM. For ease of visualization, $y$-axis is in log scale, except in $[-0.1, 0.1]$ where it is linear. All contradiction instances are excluded.
    }
    \label{fig:learned-coefficients-paired-noc}
\end{figure}

\begin{figure}[t!]
    \centering
    \includegraphics[width=0.48\textwidth]{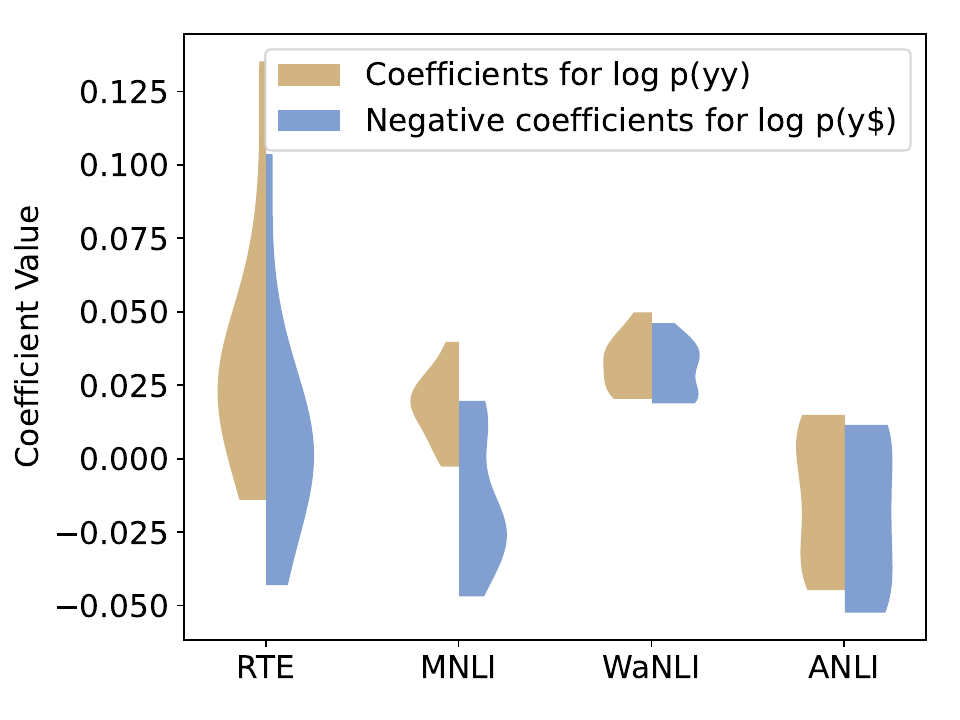}
    \caption{
    The RHS coefficients, for $\log p(y\$)$ and $\log p(yy)$, marginalized across all LMs. All contradiction instances are excluded.
    }
    \label{fig:learned-coefficients-yy-vs-y-noc}
\end{figure}

\begin{figure*}[t!]
    \centering
    \includegraphics[width=0.48\textwidth]{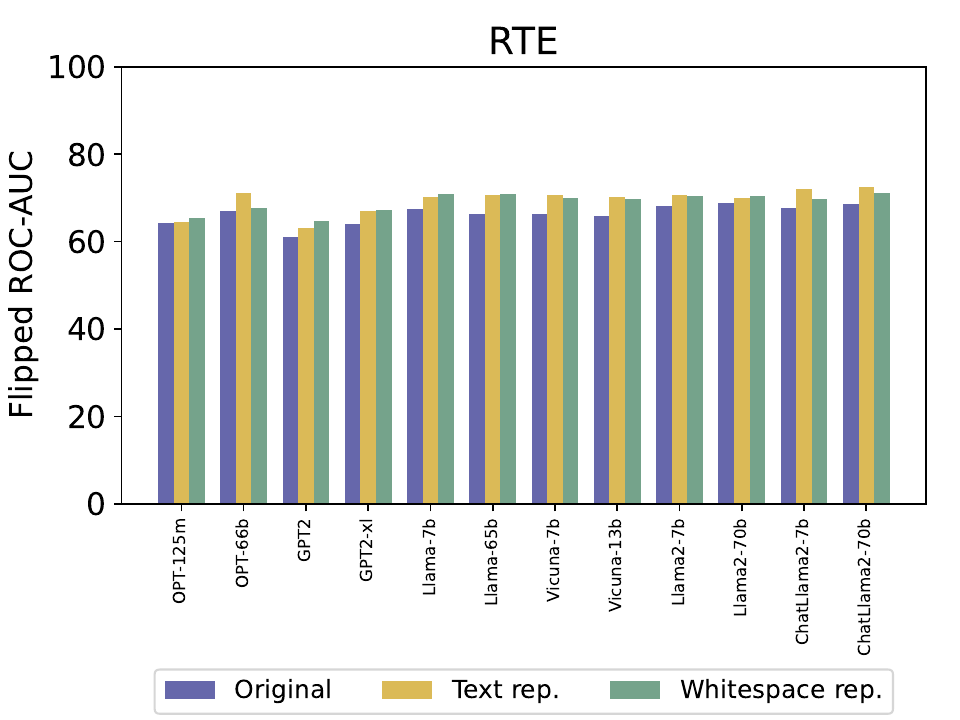}
    \includegraphics[width=0.48\textwidth]{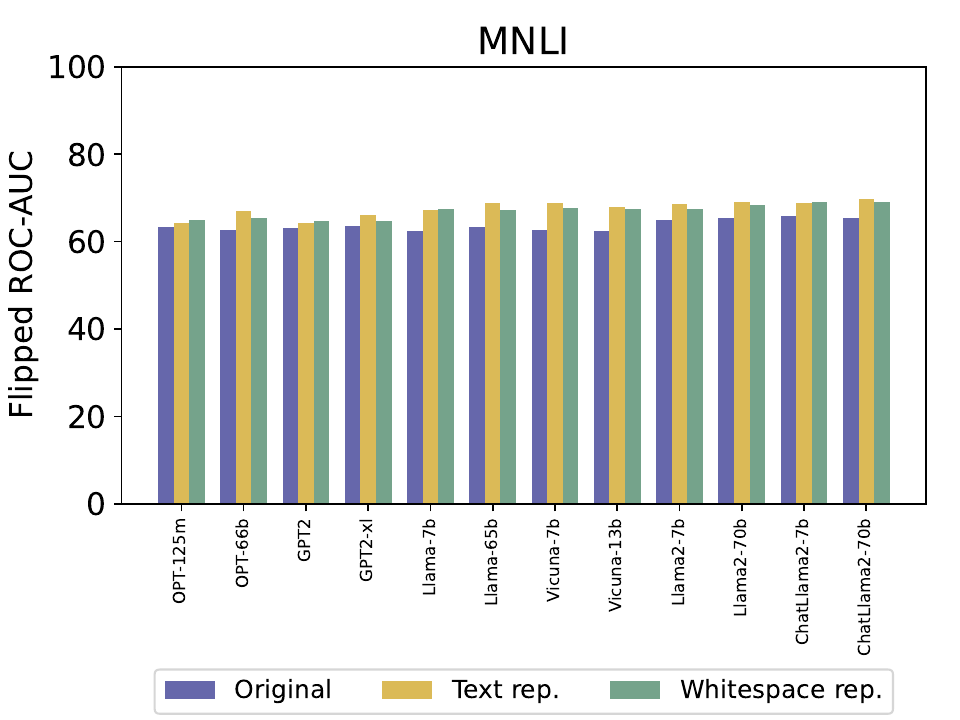}
    \caption{Performance of noise-tolerant (\S\ref{sec:noisy-speaker}) vs.~original test on RTE training set and MNLI matched validation set. All contradiction instances are excluded.}
    \label{fig:repetition-noc}
\end{figure*}

\begin{figure}[t!]
    \centering
    \includegraphics[width=0.48\textwidth]{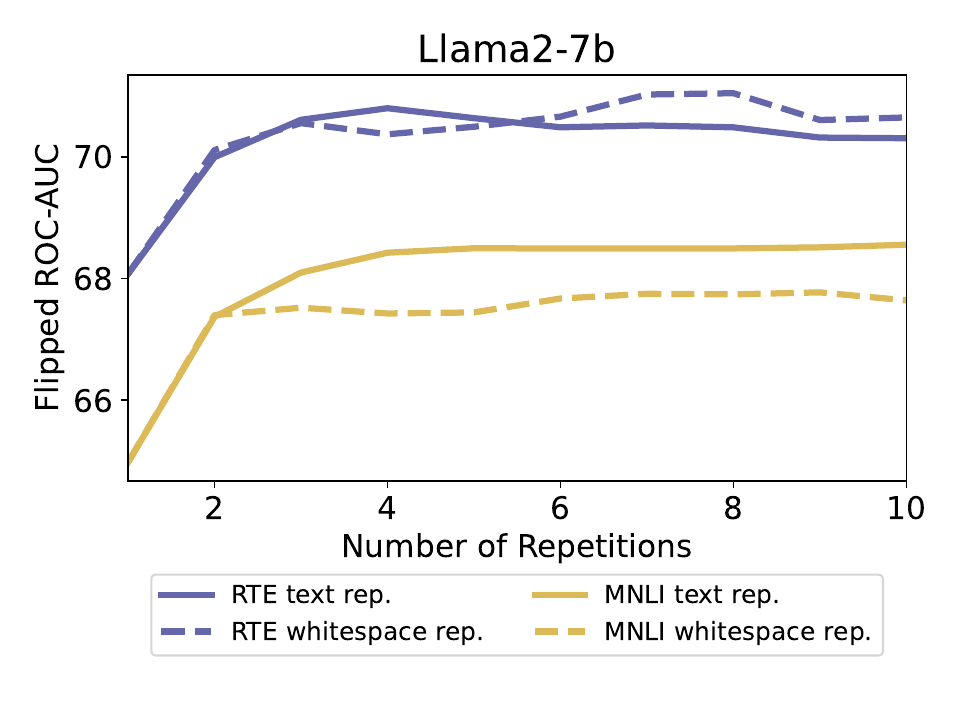}
    \caption{Performance of the noise-tolerance test with different numbers of repetitions (values of $n$ in \Cref{eq:repetition}). The original test is $n=1$. All contradiction instances are excluded.}
    \label{fig:num-repetitions-noc}
\end{figure}

\begin{figure*}[t!]
    \centering
    \includegraphics[width=\textwidth]{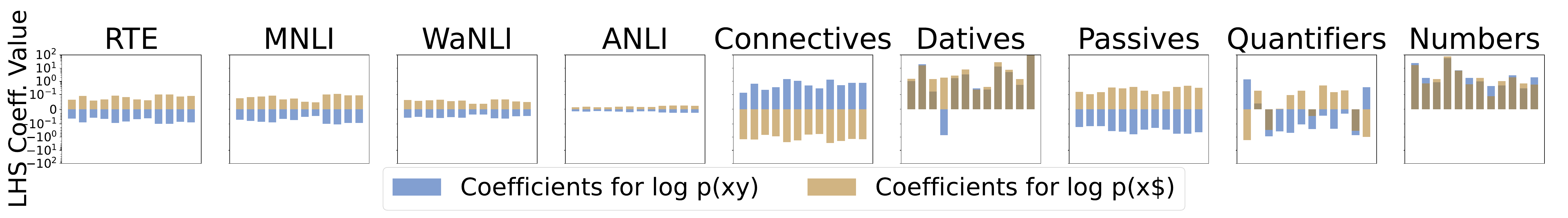}
    \includegraphics[width=\textwidth]{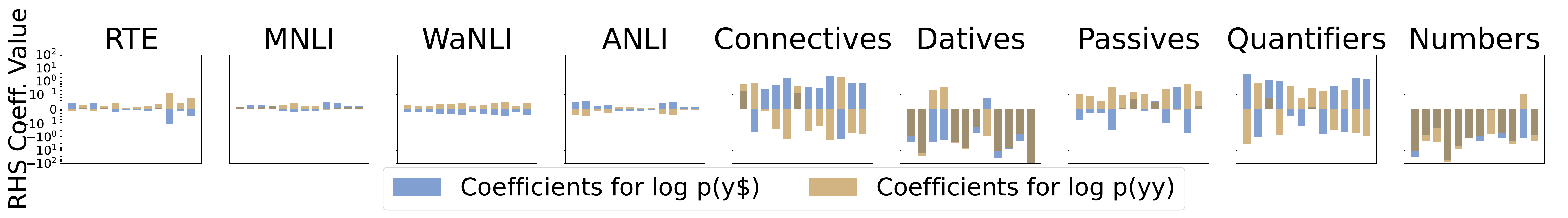}
    \caption{Learned logistic regression coefficients for the log-probability features for the broad-coverage datasets. Each bar represents one LM. For ease of visualization, $y$-axis is in log scale, except in $[-0.1, 0.1]$ where it is linear; and it is capped at $[-100, 100]$, requiring truncation in a few cases. All contradiction instances are excluded.
    }
    \label{fig:all-learned-coefficients-paired-noc}
\end{figure*}

\end{document}